\documentclass[journal]{IEEEtran}
\ifCLASSINFOpdf
\else
\fi
\pdfoutput=1
\usepackage{graphicx}
\usepackage[tight]{subfigure}
\usepackage{multirow}
\usepackage{amsmath}
\usepackage{amsthm}
\usepackage{amssymb}
\usepackage{amsfonts}
\usepackage{booktabs}
\usepackage{chngpage}
\usepackage[table,xcdraw]{xcolor}
\usepackage{authblk}
\usepackage{setspace}
\usepackage{cite}
\usepackage{textcomp}
\usepackage{color}
\usepackage{algorithmic}
\usepackage{grffile}
\usepackage{float}
\usepackage[linesnumbered,titlenumbered,ruled,vlined,commentsnumbered,noend]{algorithm2e}
\usepackage{capt-of}
\usepackage{caption}
\DeclareMathOperator*{\argmin}{argmin}
\newtheorem{definition}{Definition}
\newtheorem{proposition}{Proposition}

\newtheorem{assumption}{Assumption}
\newtheorem{theorem}{Theorem}

\newcommand{\AlgName}{\textsc{Ubar}\xspace}
\newcommand{\bheading}[1]{{\vspace{2pt}\noindent{\textbf{#1}}\hspace{2pt}}}

\newcommand{\newtextC}[1]{{\leavevmode\color{black}#1}}

\begin{document}
\title{Byzantine-resilient Decentralized Stochastic Gradient Descent}
\author{Shangwei Guo, Tianwei Zhang, Han Yu,  Xiaofei Xie, Lei Ma, Tao Xiang, and Yang Liu
    \thanks{S. Guo and T. Xiang are with College of Computer Science, Chongqing University, Chongqing 400044, China (email: \{swguo, txiang\}@cqu.edu.cn).}
    \thanks{T. Zhang, H. Yu, X. Xie, and Y. Liu are with School of Computer Science and Engineering, Nanyang Technological University 639798, Singapore (email: \{tianwei.zhang, han.yu, xfxie, and yangliu\}@ntu.edu.sg).}
    \thanks{L. Ma is with University of Alberta, Edmonton, Alberta T6G 2R3, Canada (email: ma.lei@acm.org).}}

\newenvironment{packeditemize}{
\begin{list}{$\bullet$}{
\setlength{\labelwidth}{8pt}
\setlength{\itemsep}{0pt}
\setlength{\leftmargin}{\labelwidth}
\addtolength{\leftmargin}{\labelsep}
\setlength{\parindent}{0pt}
\setlength{\listparindent}{\parindent}
\setlength{\parsep}{0pt}
\setlength{\topsep}{3pt}}}{\end{list}}

\maketitle

\begin{abstract}
Decentralized learning has gained great popularity to improve learning efficiency and preserve data privacy. Each computing node makes equal contribution to collaboratively learn a Deep Learning model. The elimination of centralized Parameter Servers (PS) can effectively address many issues such as privacy, performance bottleneck and single-point-failure. However, how to achieve Byzantine Fault Tolerance in decentralized learning systems is rarely explored, although this problem has been extensively studied in centralized systems.

In this paper, we present an in-depth study towards the Byzantine resilience of decentralized learning systems with two contributions. First, from the adversarial perspective, we theoretically illustrate that Byzantine attacks are more dangerous and feasible in decentralized learning systems: even one malicious participant can arbitrarily alter the models of other participants by sending carefully crafted updates to its neighbors. Second, from the defense perspective, we propose \AlgName, a novel algorithm to enhance decentralized learning with Byzantine Fault Tolerance. Specifically, \AlgName provides a \textbf{U}niform \textbf{B}yzantine-resilient \textbf{A}ggregation \textbf{R}ule for benign nodes to select the useful parameter updates and filter out the malicious ones in each training iteration. It guarantees that each benign node in a decentralized system can train a correct model under very strong Byzantine attacks with an arbitrary number of faulty nodes. We conduct extensive experiments on standard image classification tasks and the results indicate that \AlgName can effectively defeat both simple and sophisticated Byzantine attacks with higher performance efficiency than existing solutions.
\end{abstract}
\IEEEpeerreviewmaketitle

\section{Introduction}\label{sec:introduction}

\IEEEPARstart{T}{he} rapid development of edge computing and Deep Learning (DL) technologies leads to the era of Artificial Intelligence of Things. Nowadays, it is a trend to learn and deploy powerful DL models on edge devices \cite{li2018learning,wang2019e2,sun2020stealthy,chen2019distributed,zhao2019dynamic} for various AI tasks (e.g., image classification, video processing).
Such collaborative learning can increase the model generalization and achieve data privacy, since the model is trained from different sources of data without being released. Meanwhile, the collaboration mode enables resource-constrained devices to train large-scale models efficiently.
One typical example is the federated learning system \cite{mcmahan2016communication,bonawitz2019towards}, where multiple edge devices can collaborate to train a shared DL model for different applications and scenarios, such as healthcare \cite{rudovic2021personalized}, security surveillance \cite{zhang2021distributed}, intelligent transportation \cite{manias2021making}.

However, federated learning introduces a centralized parameter server, which can bring new security and efficiency
drawbacks \cite{lian2017can,bonawitz2019towards,xie2019diffchaser}. First, federated learning  suffers from single point of failure. The functionality
of the system highly depends on the operations of the parameter server. If the server gets crashed or
hacked, then the entire system will be broken down, affecting all the edge devices. Second, the centralized parameter server can be the performance bottleneck, particularly when a large amount of edge devices are connected to this server. 

Due to the limitations of PS-based centralized learning, there is a growing trend towards training a DL model in a decentralized fashion \cite{yang2016rd,su2016fault,lian2017can,dobbe2017fully,tang2018communication,lalitha2019decentralized}. Specifically, centralized servers are eliminated from the system while each participant plays an equal role (both training and aggregating parameters) in learning the model \cite{tsitsiklis1984problems,nedic2009distributed}. This decentralization mode exhibits huge potential for DL applications in many scenarios: in autonomous driving \cite{chen2015robust,gulati2018deep}, cars can capture images or videos during the driving and collaboratively learn powerful models for detecting traffic lights, sign, lane and pedestrians; in video coding, users can learn faster and better video coding mechanism by communicating with others \cite{abou2012fusion,yang2018high}.

A distributed system can be threatened by the famous Byzantine Generals Problem \cite{lamport1982byzantine}:
some nodes inside the network can conduct inappropriate behaviors, and propagate wrong information, leading to
the failure of the entire system. This is particularly dangerous in the distributed learning scenarios due to three reasons. (1)
Distributed learning requires the collaboration of thousands of edge devices from different domains and
parties. It is impossible to guarantee that each device is trusted and reliable. A single dishonest node
can send wrong parameters/estimates to affect the entire network and final results. (2) Modern IoT devices and networks are complicated and vulnerable. Recent years have witnessed many infamous IoT attacks (e.g., Mirai Botnet
\cite{antonakakis2017understanding}, Stuxnet \cite{langner2011stuxnet}), enabling an adversary to easily
compromise a large scale of IoT devices. This facilitates the Byzantine attacks in distributed learning
systems. (3) The consequences of attacks can be very severe. Past works have shown that an adversary can
compromise the centralized distributed learning system to alter the behaviors of the training process or final models
\cite{blanchard2017machine}.

This Byzantine Generals Problem has been extensively studied in the centralized PS-based learning systems. Attacks with different threat models and goals \cite{xie2019zeno,fang2019local} were designed to demonstrate this vulnerability. Meanwhile, Byzantine-resilient defense solutions were also introduced to enhance the system. However, very few works have focused on the Byzantine threats in decentralized learning systems. We are particularly interested in two questions: (1) \emph{how feasible and severe are the Byzantine attacks in decentralized learning systems?} (2) \emph{How can we improve the Byzantine resilience of a decentralized system?}. Currently there are no satisfactory answers due to the distinct features of centralized and decentralized systems.

In this paper, we provide an in-depth study to answer the above two questions. First, we formally define the Byzantine Generals Problem in the decentralized learning setting, and theoretically analyze the corresponding vulnerabilities. We discover that the indirect connection to a malicious node cannot reduce the attack cost and amplify the damage. We prove that an adversary can just use one node to alter the models of all nodes inside the system arbitrarily. This is different from the centralized system, which only requires tampering the model on PS for a successful attack.

Second, we explore the possible solutions to secure decentralized learning systems with Byzantine Fault Tolerance (BFT). It is challenging to apply the Byzantine-resilience methods from PS-based systems \cite{blanchard2017machine,xie2018generalized,chen2018draco,xie2019slsgd,konstantinov2019robust,sohn2020election,konstantinov2020sample} to the decentralized scenario due to two reasons.
First, those defenses have security and efficiency drawbacks in protecting centralized systems. They are either vulnerable to elaborately designed Byzantine attacks \cite{mhamdi2018hidden,baruch2019little,fang2019local}, or have large computation overhead and scalability issue with unrealistic requirements (e.g., the PS has extra validation dataset) \cite{xie2019zeno,fang2019local}.
These limitations still exist if the defenses were extended into decentralized systems. The second reason lies in the huge differences between centralized and decentralized learning systems. Existing defenses are mainly designed for the centralized PS to make decisions. However, each participant in decentralized learning acts as not only a worker node, but also a PS. In addition, the number of neighbors connected to each node varies dramatically. So some assumptions made in the centralized defenses will not hold. To the best of our knowledge, currently there are few research papers \cite{yang2019byrdie} attempting to achieve Byzantine-resilient decentralized learning by comparing the distances among estimates, which is vulnerable to sophisticated Byzantine attacks, as demonstrated in the evaluation section of this paper.

We propose \AlgName, a novel \textbf{U}niform \textbf{B}yzantine-resilient \textbf{A}ggregation \textbf{R}ule to secure decentralized learning systems.  \AlgName consists of two design stages. The first stage is introduced to mitigate simple Byzantine attacks (e.g., \cite{xie2018generalized}) by shortlisting a set of candidate nodes: each benign node selects a number of potential benign nodes based on the distances of their parameters to its own. The second stage is used to select the final parameters and defeat advanced Byzantine attacks \cite{mhamdi2018hidden,baruch2019little,fang2019local}: each benign node uses its training samples to test the performance of the parameters from the first phase, and chooses the ones with the best training quality.

\AlgName leverages the unique features of decentralized systems to overcome the limitations of prior solutions. Since each node acts as both a worker and PS, it can use its own parameters as the baseline (Stage 1) instead of the average or median of neighbor nodes in PS-based systems. This can effectively protect the baseline values from being manipulated, and mitigate an arbitrary number of Byzantine nodes. Each node also uses its training samples for performance evaluation (Stage 2), which can perfectly relax the unrealistic assumption of the server's availability of validation datasets in PS-based systems. Besides, since Stage 1 is vulnerable to advanced Byzantine attacks \cite{mhamdi2018hidden,baruch2019little,fang2019local} while Stage 2 suffers from scalability and cost issues, the integration of these two stages can achieve both efficiency and strong Byzantine resilience.
We conduct comprehensive experiments to show that \AlgName is tolerant against both simple and sophisticated attacks, while all
existing defense solutions fail. Besides, \AlgName also achieves 8-30X performance improvement over existing methods.

The key contributions of this paper are:
\begin{packeditemize}
    \item We theoretically analyze and demonstrate the vulnerabilities of the Byzantine Generals Problem in decentralized learning.
    \item We propose \AlgName, a uniform Byzantine-resilient aggregation rule, to defeat an arbitrary number of Byzantine nodes in decentralized systems.
    \item We conduct extensive experiments to show \AlgName outperforms other solutions for both security and performance.
\end{packeditemize}

The rest of this paper is organized as follows. Background and related works are reviewed in Section~\ref{sec:literature}.
Section~\ref{sec:problemdef} gives formal definitions of decentralized systems
We analyze the Byzantine Fault of decentralized learning in Section~\ref{sec:attack}. Section~\ref{sec:algorithm} presents our novel Byzantine-resilient solution.
Section~\ref{sec:experiments} shows the experimental results under various attacks and system settings. Section~\ref{sec:disscussion} and \ref{sec:conclusion} discusses the limitations and concludes the paper.

\section{Background and Related Work}\label{sec:literature}

\subsection{Byzantine-resilient Centralized Learning}
A centralized learning system consists of a Parameter Server (PS) and multiple distributed
worker nodes, as shown in Figure \ref{fig:framework}(a). Every worker node has its own training dataset, but adopts the same training
algorithm. \newtextC{In each iteration, a worker node 1) pulls the gradient from the PS, 2) updates the gradient based on its local data, 3) uploads the new gradient to the PS, and
the PS 4) aggregates all the received gradients from the worker nodes into one gradient vector.} The nodes repeat the above steps from the new gradient, until the training process
is terminated and a model is produced.

\begin{figure}[t]
	\captionsetup{justification=justified}
	\centering
	\includegraphics[width=0.9\columnwidth]{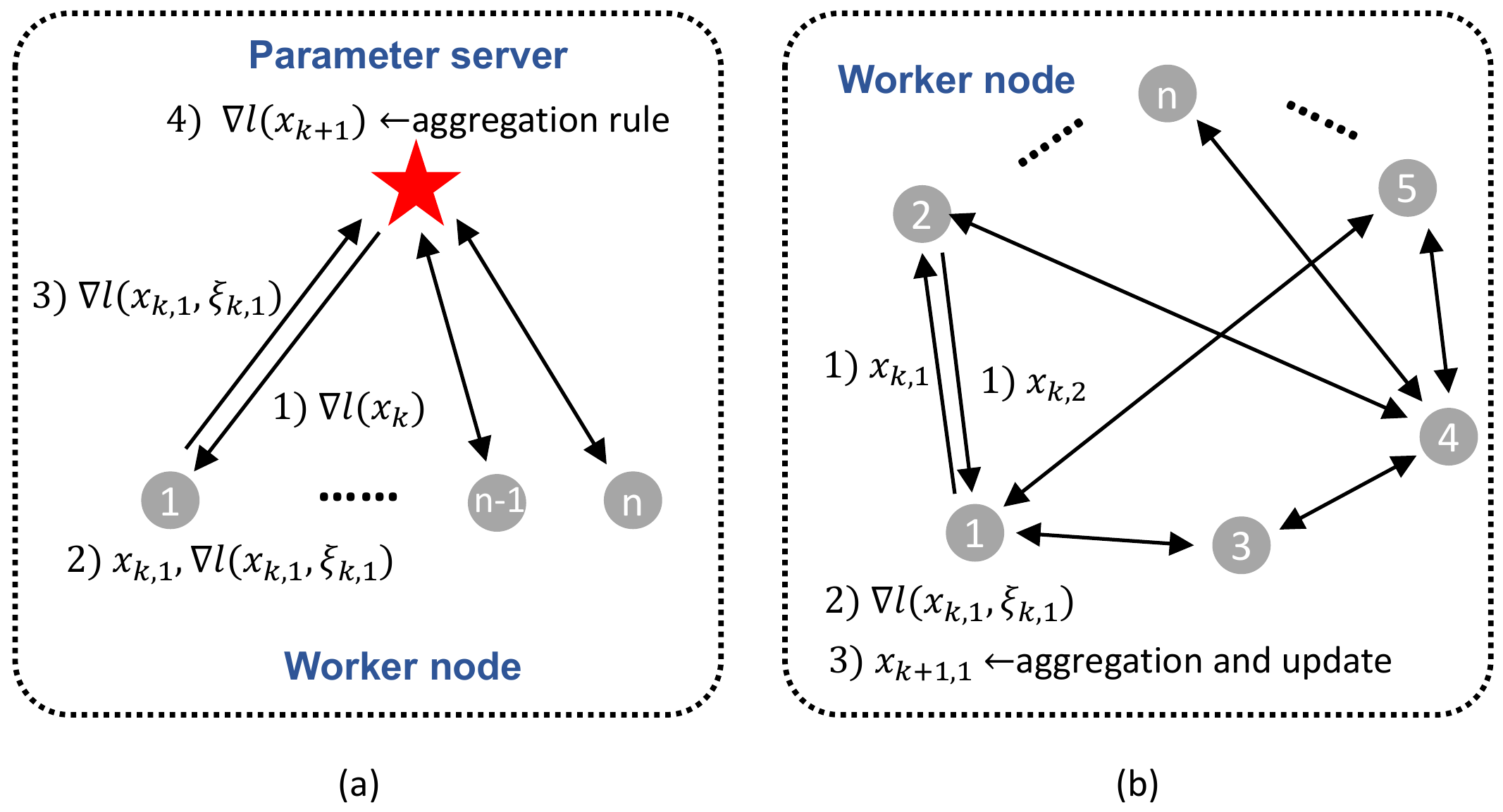}\label{fig:framework_decen}
	\caption{Distributed learning in centralized (a) and decentralized (b) fashions.}
	\label{fig:framework}
\end{figure}
Dishonest nodes can compromise the training process and the final model by uploading
wrong gradients \cite{mhamdi2018hidden,baruch2019little,fang2019local}. It is necessary for
the PS to detect such Byzantine nodes and discard their updates when aggregating
the gradients.

Motivated by the parameter difference between benign and malicious estimates, a number of solutions cluster the uploaded gradients and detect the outliers based on the vector distances. For instance, Blanchard et al. proposed Krum \cite{blanchard2017machine},
which chooses the gradient vector with the minimal sum of squared distances to its neighbors as the aggregated one. Median-based Aggregation rules \cite{xie2018generalized,yin2018byzantine} were designed, which inspect the gradient vectors and calculate the median values in each dimension to defeat Byzantine
attacks. Mhamdi et al. introduced Bulyan  \cite{mhamdi2018hidden}, to further enhance existing
Byzantine-resilient aggregation rules by combining Krum and Median-based aggregation rules. Although these defenses can defeat
simple Byzantine attacks such as Gaussian and bit-flip attacks \cite{blanchard2017machine,xie2018generalized}, they were vulnerable against
more sophisticated attacks \cite{baruch2019little,fang2019local}. The reason lies in the vulnerability of distance-based strategies: the close distance between two gradients does not imply similar performance. Thus, these sophisticated attacks could create gradients that are malicious but indistinguishable from benign gradients in distance.

Some solutions selects the benign nodes by evaluating the performance of each uploaded
gradient on extra validation datasets. For instance, Xie et al. proposed Zeno \cite{xie2019zeno} and Zeno++ \cite{xie2019zeno++} for synchronous and asynchronous learning systems, respectively. Both Zeno and Zeno++ calculate the prediction accuracy of each gradient on the extra validation datasets to identify Byzantine nodes.  However,
they require that the PS has a validation dataset, which is not realistic under some circumstances.
Besides, performance evaluations of all gradients have much more overhead than parameter evaluations. This
can significantly increase the total training time and the computation burdens for the PS,
especially when the number of worker nodes is larger.

Some solutions select benign gradients and nodes based on their history records. For instance, Hidden Markov Model was utilized \cite{munoz2019byzantine} to learn the quality of parameter updates during distributed training. The learned profiles can improve the efficiency and accuracy of detecting malicious nodes.  Pan et al. \cite{panjustinian} utilized the historical interactions with the workers as experience to identify Byzantine attacks via reinforcement learning techniques. However, these solutions cannot guarantee Byzantine resilience. An adversary can easily bypass the detection algorithms by pretending to be benign at the beginning and only uploading malicious parameters at the last several iterations. Then the learned profiles cannot predict malicious behaviors in future iterations.

\subsection{Byzantine-resilient Decentralized Learning}
Decentralized learning systems remove the PS, as every node in the network is also responsible
for model update \cite{he2018cola,nedic2009distributed,lian2017can}. The architecture of a decentralized learning system is illustrated in Figure \ref{fig:framework}(b). Specifically, in each iteration of the training process, each worker node 1) broadcasts its parameter vectors (estimates\footnote{``Parameter vector'' and ``estimate''
are used interchangeably.}) to its neighbor nodes, and receive the estimates from them; 2) trains the model
estimates using the local data. 3) It then aggregates them with the neighbor nodes' estimates and updates the model.

Compared to centralized learning, research of Byzantine-resilient decentralized learning is still at an early
stage. Several attempts have been made to achieve Byzantine-resilient decentralized learning \cite{yang2019byrdie,yang2019bridge,peng2020byzantine}. For example, Yang et al. proposed ByRDiE \cite{yang2019byrdie} and BRIDGE \cite{yang2019bridge}, which simply apply
the trimmed-median algorithm from centralized systems \cite{xie2018generalized,yin2018byzantine} to decentralized systems. While ByRDiE is designed for the coordinate descent optimization algorithm, BRIDGE is used in decentralized learning systems with SGD. Similar to \cite{xie2018generalized,yin2018byzantine}, those solutions are vulnerable to some Byzantine attacks \cite{mhamdi2018hidden,baruch2019little}. In Section~\ref{sec:experiments}, we will demonstrate their incapability of defeating sophisticated Byzantine attacks.

Yang and Bajwa \cite{yang2016rd} proposed RD-SVM to support distributed Support Vector Machine (SVM) against Byzantine attacks. RD-SVM compares the losses from neighbor nodes to identify and filter potential Byzantine nodes. It adopts the hinge loss, and involves all the data samples at each iteration for Byzantine identification. Hence, it is more applicable to linear classifiers like SVM. In contrast, this paper focuses on deep learning models with the mainstream SGD algorithm and batch training feature. We propose \AlgName to fulfill these requirements and improving the efficiency.

\section{System Model of Decentralized Learning}~\label{sec:problemdef}
In this section, we formally define a decentralized communication system and describe the learning task.
\subsection{Decentralized Systems}
A decentralized system is defined as an undirected graph: $\mathcal{G} = (V, E)$, where $V$ denotes a set of $n$ nodes
and $E$ denotes a set of edges representing communication links. Specifically, we have
\begin{packeditemize}
    \item $(i, j) \in E$ if and only if node $i$ can receive information from node $j$;
    \item $(j, i) \in E$ if $(i, j) \in E$.
\end{packeditemize}

Let $\mathcal{N}_i = \{j| (i, j) \in E\}$ be the set of the neighbors of node $i$. We further assume that $n_b$ out of $n$ nodes are benign and the rest are malicious. We can define a subgraph that only contains the benign nodes:

\begin{definition}(Benign Induced Subgraph)
    The benign induced subgraph, $\mathcal{G}_b = (V_b, E_b)$, is a subgraph of $\mathcal{G}$, formed by all the benign nodes in $\mathcal{G}$ and all the edges connecting those benign nodes. Specifically,
    \begin{packeditemize}
    \item \newtextC{$ i \in V_b \subseteq V$ if $i$ is a benign node and $|V_b|=n_b$;}
    \item $(i, j) \in E_b \subseteq E$ if and only if $i, j \in V_b$;
    \item $(j, i) \in E_b$ if $(i, j) \in E_b$.
    \end{packeditemize}
\end{definition}

Following the information exchange models in \cite{tsitsiklis1984problems,nedic2009distributed}, we assume the benign induced
subgraph is fully connected, i.e., giving two arbitrary benign nodes $i$ and $j$, there always exists at least one path that
connects these two nodes. We formally state the assumption as below:

\begin{assumption}\label{label:connectivity}{(Connectivity of Benign Induced Subgraph)}
    There exists an integer $\tau$ such that for $\forall i, j \in V_b$, node $j$ can propagate its information to node $i$ through at most $\tau$ edges.
\end{assumption}


\subsection{Model Training}
In a decentralized learning system, $n$ nodes cooperatively train a model by optimizing the loss function with SGD and exchanging estimates with their neighbors.
Let $x \in \mathbb{R}^d$ be the $d$-dimensional estimate vector of a DL model; $l$ be the loss function. Each node $i \in V$ obtains a training dataset $D_i$, consisting of independent and identically distributed (IID) data samples from a distribution $D$. Those \newtextC{$n$} nodes train a shared model by solving the following optimization problem.
\begin{equation}
    \min_{x\in \mathbb{R}^d} \mathbb{E}_{\xi\sim D}l(x,\xi)
\end{equation}
where $\xi$ is a training data sample from $D$ and $l(x,\xi)$ is calculated on $\xi$.
\section{Byzantine Attack in Decentralized Learning}~\label{sec:attack}
In this section, we theoretically demonstrate the feasibility and severity of Byzantine attacks in decentralized learning systems.
Following the decentralized network topology, the nodes in $V$ iteratively optimize the shared model until reaching convergence or the maximum number of iterations. Specifically, at the $k$-th iteration, node $i$ has its local estimate denoted as $x_{k, i}$, and broadcasts it to its neighbors. When receiving the estimates from the neighbors, node $i$ will update its local estimate according to the General Update Function (GUF):
\begin{definition}{(GUF)}
    Let $x_{k,i}$, $\nabla l(x_{k,i},\xi_{k,i})$ be the estimate and gradient of node $i$ at the $k$-th iteration. $\{x_{k,j}, j \in \mathcal{N}_i\}$ are the estimates from its neighbors. $\mathcal{R}$ is an aggregation rule. Node $i$ updates its estimate for the $(k+1)$-th iteration using the following general update function:
    \begin{equation}
        x_{k+1,i} = \alpha x_{k,i} + (1-\alpha)\mathcal{R}(x_{k,j}, j \in \mathcal{N}_i)- \lambda \nabla l(x_{k,i},\xi_{k,i})
    \end{equation}
    where $\lambda$ is the learning rate; $\alpha$ is a hyper-parameter that balances the weights of the estimates.
\end{definition}
Without loss of generality, we assume all the nodes have the same learning rate. The stochastic gradient can be replaced with a mini-batch of stochastic gradients \cite{lian2017can}.

A straightforward and common way is the average aggregation rule:
\begin{equation}\label{equ:original_ar}
    \mathcal{R}_{Average} = \frac{1}{|\mathcal{N}_i|}\sum_{j\in \mathcal{N}_i}x_{k,j}
\end{equation}
and $\alpha$ is set as $\frac{1}{|\mathcal{N}_i|+1}$. However, because the average aggregation in Equation \ref{equ:original_ar} does not consider BFT, this training
process can be easily compromised by Byzantine attacks: an adversary can use just one malicious node to send wrong estimates to
its neighbors and alter their aggregated estimates.
More seriously, due to the fully connectivity of benign induced
subgraph (Assumption \ref{label:connectivity}), this fault will be also propagated to other benign nodes not directly connected to this Byzantine node after several
iterations, and finally all the nodes in this network will be affected.

\begin{theorem}\label{lemma:attack}
    Consider a decentralized system \newtextC{under the average aggregation rule $\mathcal{R}_{Average}$}. In this system $\hat{i}$ is a Byzantine node, attempting to add a malicious vector $\hat{x}$ to the estimate of a benign node $i_{\tau'}$. The shortest distance (i.e., number of edges) between them is $\tau'$ and $\{i_s\}_{s=1}^{\tau'-1}$ are the benign nodes on the shortest trace between $i_{\tau'}$ and $\hat{i}$. The distance between node $i_s$ and $\hat{i}$ is $s$. Then at the $k_0$-th iteration, node \newtextC{$\hat{i}$} can broadcast to its neighbors the following estimate to achieve this goal in $\tau'$ iterations:
    \begin{equation}
        x = x_{k_0, \hat{i}} + \hat{x}\prod_{s=1}^{\tau'}(|\mathcal{N}_{i_s}|+1)
    \end{equation}
    where $|\mathcal{N}_{i_s}|$ is the number of neighbors of node $i_s$.
\end{theorem}
\begin{proof}
    We assume that there is only one path of $\tau'$ edges from node $\hat{i}$ to $i_{\tau'}$. We prove the theorem by mathematical induction.

    If $\tau' = 1$, the two nodes are neighbors. Then, the estimate of node $i_1$ at $k_0 + 1$ iteration is
    \begin{align}
        &\hat{x}_{k_0+1, i_1} =  \frac{1}{|\mathcal{N}_{i_1}|+1}(x_{k_0, i_1} + \sum_{j\in \mathcal{N}_{i_1}/{\hat{i}}}x_{k_0,j} + x_{k_0, \hat{i}} + \hat{x}(|\mathcal{N}_{i_1}| \\\nonumber
                                & +1))- \lambda \nabla l(x_{k_0,i_1},\xi_{k_0,i_1}) \\
                    &= \frac{1}{|\mathcal{N}_{i_1}|+1}(x_{k_0, i_1} + \sum_{j\in \mathcal{N}_{i_1}}x_{k_0,j}) - \lambda \nabla l(x_{k_0,i_1},\xi_{k_0,i_1}) + \hat{x}
    \end{align}
    It proves that the theorem is true when $\tau' = 1$.

    We now assume the theorem is true when $\tau' = k$, i.e. node $\hat{i}$ sends
    \begin{equation}
        \hat{x}_{k_0, \hat{i}} = x_{k_0, \hat{i}} + \hat{x}\prod_{s=1}^{k}(|\mathcal{N}_{i_s}|+1).
    \end{equation}
    to its neighbors at $k_0$-th iteration. Then, at $(k_0 +k)$-th iteration, node $i_k$'s estimate is
    \begin{equation}
        \hat{x}_{k_0+k, i_k} = x_{k_0+k, i_k} + \hat{x}
    \end{equation}
    where $x_{k_0+k, i_k}$ is the benign estimate that node $\hat{i}$ should send to its neighbors when it was not controlled.

    Consider the $k+1$ case. Node $\hat{i}$ sends
    \begin{equation}
        \hat{x}_{k_0, \hat{i}} = x_{k_0, \hat{i}} + \hat{x}\prod_{s=1}^{k+1}(|\mathcal{N}_{i_s}|+1)
    \end{equation}
    to its neighbors at the $k_0$-th iteration. At the $(k_0 +k)$-th iteration, the estimate of node $i_k$ is
    \begin{equation}
        \hat{x}_{k_0+k, i_k} = x_{k_0+k, i_k} + \hat{x}(|\mathcal{N}_{i_{k+1}}|+1).
    \end{equation}

    Then, at the $(k_0 +k+1)$-th iteration, the estimate of node $x_{k+1}$ is affected:
    \begin{align}\nonumber
        &\hat{x}_{k_0+k+1, i_{k+1}}\\\nonumber
        &=  \frac{1}{|\mathcal{N}_{i_{k+1}}|+1}(x_{k_0+k, i_{k+1}} + \sum_{j\in \mathcal{N}_{i_{k+1}}/{i_k}}x_{k_0+k,j} \\\nonumber
        &+ x_{k_0+k, i_k} + \hat{x}(|\mathcal{N}_{i_{k+1}}|+1)) - \lambda \nabla l(x_{k_0+k,i_{k+1}},\xi_{k_0+k,i_{k+1}}) \\\nonumber
                    &= \frac{1}{|\mathcal{N}_{i_{k+1}}|+1}(x_{k_0+k, i_{k+1}} + \sum_{j\in \mathcal{N}_{i_{k+1}}}x_{k_0+k,j}) \\
                    &- \lambda \nabla l(x_{k_0+k,i_{k+1}},\xi_{k_0+k,i_{k+1}}) + \hat{x}
    \end{align}
\end{proof}
\section{Byzantine-resilient Solution}\label{sec:algorithm}

\subsection{Byzantine-resilient Aggregation Rule.}

Due to the Byzantine threat of decentralized learning, it is necessary to design a robust aggregation rule to defeat Byzantine nodes. This rule
should guarantee that all benign nodes converge to the optimal estimate learned without Byzantine nodes.
In the following, we propose \AlgName, a novel aggregation rule for decentralized systems to satisfy the above requirement and uniformly defend against Byzantine attacks.

\subsection{\AlgName}
The design of \AlgName is motivated by three observations. First, as introduced in Section \ref{sec:literature}, existing Byzantine defenses for centralized systems have
certain security vulnerabilities or practical limitations. Such design flaws
still exist when we extend the solutions to decentralized scenarios. Second, a
decentralized system has higher convergence requirement than a centralized system:
convergence of one parameter server enforced by the solutions cannot guarantee
the convergence of all benign nodes in a decentralized system. Third, centralized
Byzantine-resilient solutions usually assume a fixed number of faulty nodes connected
to the parameter server, while in a decentralized system, the number of faulty nodes
connected to each benign node varies significantly. As such, it is necessary to have
a more robust Byzantine-resilient solution that can defeat \emph{an arbitrary number} of
Byzantine nodes and guarantee convergence of \emph{each benign node} in decentralized
systems.

\begin{figure}[t]
	\centering
	\subfigure{\includegraphics[width=\columnwidth]{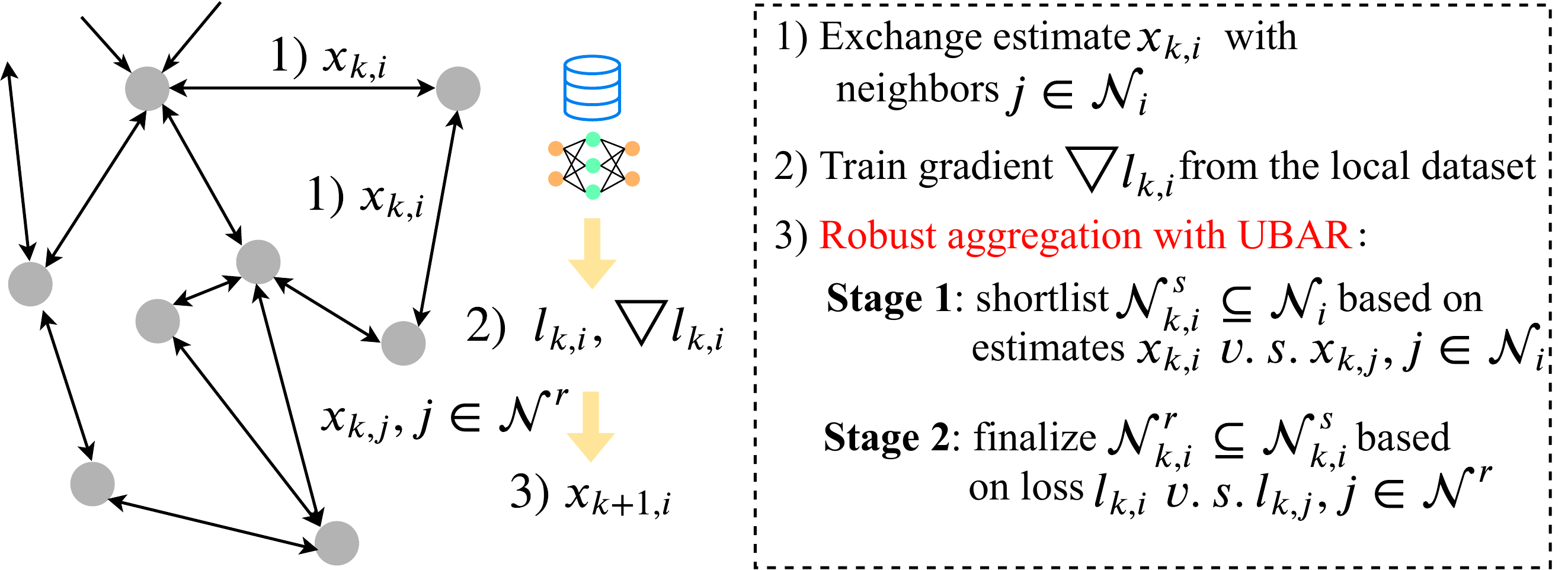}}
	\vspace{-7pt}
	\caption{A decentralized learning system with \AlgName.}
	\label{fig:ubar}
\end{figure}

\AlgName aims to achieve this goal and overcome the above limitations. It consists of two stages for each
training iteration as shown in Fig. \ref{fig:ubar}. At the first stage, each benign node selects a candidate pool of potential benign nodes from its neighbors. The selection is made by comparing the Euclidean distance of the estimate of each neighbor node with its own estimate. One innovation of this stage is the benign node uses its own parameter as the baseline value instead of the median or mean value of its neighbors' parameters as in centralized PS-based systems \cite{blanchard2017machine,xie2018generalized,yin2018byzantine}. This is based on the unique feature of decentralized systems that each node is responsible for both training and aggregation. It gives stronger Byzantine resilience as the baseline values trained from local datasets can never be manipulated by Byzantine nodes, while the aggregated parameter can be poisoned according to Theorem \ref{lemma:attack} Although after this stage, the candidate pool might still contain Byzantine nodes as the distance-based strategies are not strict Byzantine-resilient, it indeed reduces the scope of benign nodes for further selection.

At the second stage, each benign
node further picks the final nodes from the candidate pool
for estimate update. It reuses the training sample as the validation set to test the performance (i.e., loss function value) of each estimate. It selects the estimates whose loss
values are smaller than its own estimate, and calculates the average of those estimates as
the final updated value. One novelty of this stage is the adoption of training samples for performance evaluation of neighbors' parameters. In contrast, prior works in centralized systems require the PS to have an extra validation dataset for evaluation, which may not be applicable in certain scenarios.

It is interesting to note that the selection criteria at Stage 1 is still vulnerable to advanced attacks \cite{mhamdi2018hidden,baruch2019little,fang2019local}, while the strategy at Stage 2 has efficiency and scalability issues especially when the connectivity is high. By integrating them into one approach, Stage 2 can help Stage 1 further defeat the advanced attacks, while Stage 1 can reduce the computation cost at Stage 2, as it decreases the size of candidate nodes for evaluation. \AlgName can be formally described as below:
\begin{definition}{(\AlgName)}
   Let $x_{k,i}$ be the estimate of node $i$ at the $k$-th iteration; $l_{k,i}$ be the loss of the estimates on the stochastically selected data sample, i.e., $l_{k,i}=l(x_{k,i},\xi_{k,i})$; $\rho_{i}$ be the ratio of benign neighbors of node $i$. The proposed \textbf{U}niform \textbf{B}yzantine-resilient \textbf{A}ggregation \textbf{R}ule, \AlgName, is define as
   \begin{equation}
       \mathcal{R}_{\AlgName} = \begin{cases}
            \frac{1}{|\mathcal{N}_{k,i}^r|}\sum_{j \in \mathcal{N}_{k,i}^r}x_{k,j}, \ \text{if} \ \mathcal{N}_{k,i}^r\neq \emptyset\\
            x_{k, j^*}, \ \text{Otherwise}
        \end{cases}
   \end{equation}
   where
   \begin{align}\nonumber
        \text{(Stage 1)} \ &\mathcal{N}_{k,i}^s = \argmin_{\substack{\mathcal{N}^*\subseteq\mathcal{N}_{i} \\ |\mathcal{N}^*| = \rho_i|\mathcal{N}_i|}} \sum_{j\in \mathcal{N}^*} ||x_{k,j} - x_{k,i}||, \\ \nonumber
        \text{(Stage 2)} \ &\mathcal{N}_{k,i}^r = \bigcup_{\substack{ j \in \mathcal{N}_{k,i}^s\\ l_{k,j} \le l_{k,i}}}j, \ \text{and} \
        j^* = \argmin_{j\in \mathcal{N}_{k,i}^s}l_{k,j}.
    \end{align}
\end{definition}

Algorithm \ref{alg:proposed} details the training process of node $i$ using \AlgName in a decentralized system. The algorithm begins with the estimate $x_{0,i} = x_0$. At the $k$-th iteration, node $i$ broadcasts its estimate to and receives the estimates from its neighbors. It stochastically selects a training data sample $\xi_{k,i}$ and calculates the loss and the gradient (Lines \ref{line:sample}-\ref{line:gradient}). Then it conducts two-stage estimate selection.
First, it calculates the Euclidean distances between $x_{k,i}$ and the estimates from its neighbors and selects $\rho_i|\mathcal{N}_i|$ neighbors with lowest distances (Lines \ref{line:distance}-\ref{line:select}).

Second, for each estimate $x_{k,j}, j \in \mathcal{N}_{k,i}^s$, node $i$ calculates the loss of $x_{k,j}$ on $\xi_{k,i}$. It chooses the estimates that have similar or better performance than that of $x_{k,i}$ (Lines \ref{line:performance}-\ref{line:select2}). Finally it calculates the average value of the selected nodes and updates the final estimate using GUF (Lines \ref{line:agg}-\ref{line:end}).
\begin{algorithm}[tb]
    \SetAlgoLined
    \SetKwInOut{Input}{Input}
    \Input{Initial estimate $x_0$, learning rate $\lambda$, number of iterations $K$, ratio of benign nodes $\rho_i$}
    \For{$k$ in $[0, K)$}{
        Broadcast $x_{k, i}$ and receive $x_{k,j}$ from $j \in \mathcal{N}_i$\;
        Stochastically sample $\xi_{k,i}$ from $D_i$ \;\label{line:sample}
        $l_{k,i} \leftarrow l(x_{k,i},\xi_{k,i})$ and compute the local gradient $\nabla l_{k,i}$ \;\label{line:gradient}
        \For{$j$ in $\mathcal{N}_i$ \label{line:distance}}{
            $d_{i,j} \leftarrow ||x_{k,i} - x_{k, j}||$\;
        }
        $\mathcal{N}_{k,i}^s \leftarrow$ $\argmin_{\substack{\mathcal{N}^*\subseteq\mathcal{N}_{i} \\ |\mathcal{N}^*| = \rho_i|\mathcal{N}_i|}} \sum_{j\in \mathcal{N}^*} d_{i,j}$ \;\label{line:select}
        \For{$j \in \mathcal{N}_{k,i}^s$ \label{line:performance}} {
            $l_{k,j} \leftarrow l_{k,i}$\;
            \If{$l_{k,i} - l_{k,j} \ge 0$}{
                append $j$ to $\mathcal{N}_{k,i}^r$ \;
            }
        }
        \If{$\mathcal{N}_{k,i}^r \ is \ \emptyset$}{
            $j^* \leftarrow \argmin_{j \in \mathcal{N}^s_{k,i}} l_{k,j}$\;
            append $j^*$ to $\mathcal{N}_{k,i}^r $ \;\label{line:select2}
        }
        $\mathcal{R}_{k,i} \leftarrow \frac{1}{|\mathcal{N}_{k,i}^r|}\sum_{j \in \mathcal{N}_{k,i}^r}x_{k,j}$ \; \label{line:agg}
        Update the local estimate $x_{k+1, i} \leftarrow \alpha x_{k, i} + (1-\alpha)\mathcal{R}_{k,i} - \lambda \nabla l_{k,i}$ \; \label{line:end}
    }
    \Return{$x_{K, i}$}
	\caption{The training algorithm for each benign node $i$ using \emph{\AlgName}.}
    \label{alg:proposed}
\end{algorithm}

\subsection{Complexity Analysis}
The training process with \AlgName is performance efficient, as proved below:

\begin{proposition}{(Cost of \AlgName)}
    The computational complexity of \AlgName is $O(|\mathcal{N}_i|d)$ for each node at each iteration, where $d$ is the dimension of the estimate vector.
\end{proposition}
\begin{proof}
   For node $i \in V_b$, at each iteration, \AlgName aggregates the received estimates with three operations. First, \AlgName selects $\rho_i|\mathcal{N}_i|$ neighbors that are closest to its current estimate. The cost is $O(|\mathcal{N}_i|d)$. Second, \AlgName calculates the loss of the selected estimates on the stochastic sample and the cost is $O(\rho_i|\mathcal{N}_i|d)$. Finally, \AlgName takes at most $O(\rho_i|\mathcal{N}_i|d)$ to aggregate the estimates with better performance. Since $\rho_i \leq 1$ for $\forall i \in V_b$, the overall computational complexity of \AlgName is $O(|\mathcal{N}_i|d)$.
\end{proof}

Compared to existing aggregation rules, the complexity of Average aggregation, Median-based and BRIDGE is $O(|\mathcal{N}_i|d)$, while Krum and Bulyan have a complexity of $O(|\mathcal{N}_i|^2d)$. So we conclude that \AlgName maintains the same performance efficiency as some solutions and performs much better than others, especially when the number of connected neighbor nodes becomes large.

\begin{table*}[t!]\centering
	\resizebox{\textwidth}{!}{
		\begin{tabular}{c@{\hskip3pt}c@{}c@{}c}
			\rotatebox{90}{$\ \ \ \ \ \ \ \ \  $MNIST}&\includegraphics[width=0.33\textwidth]{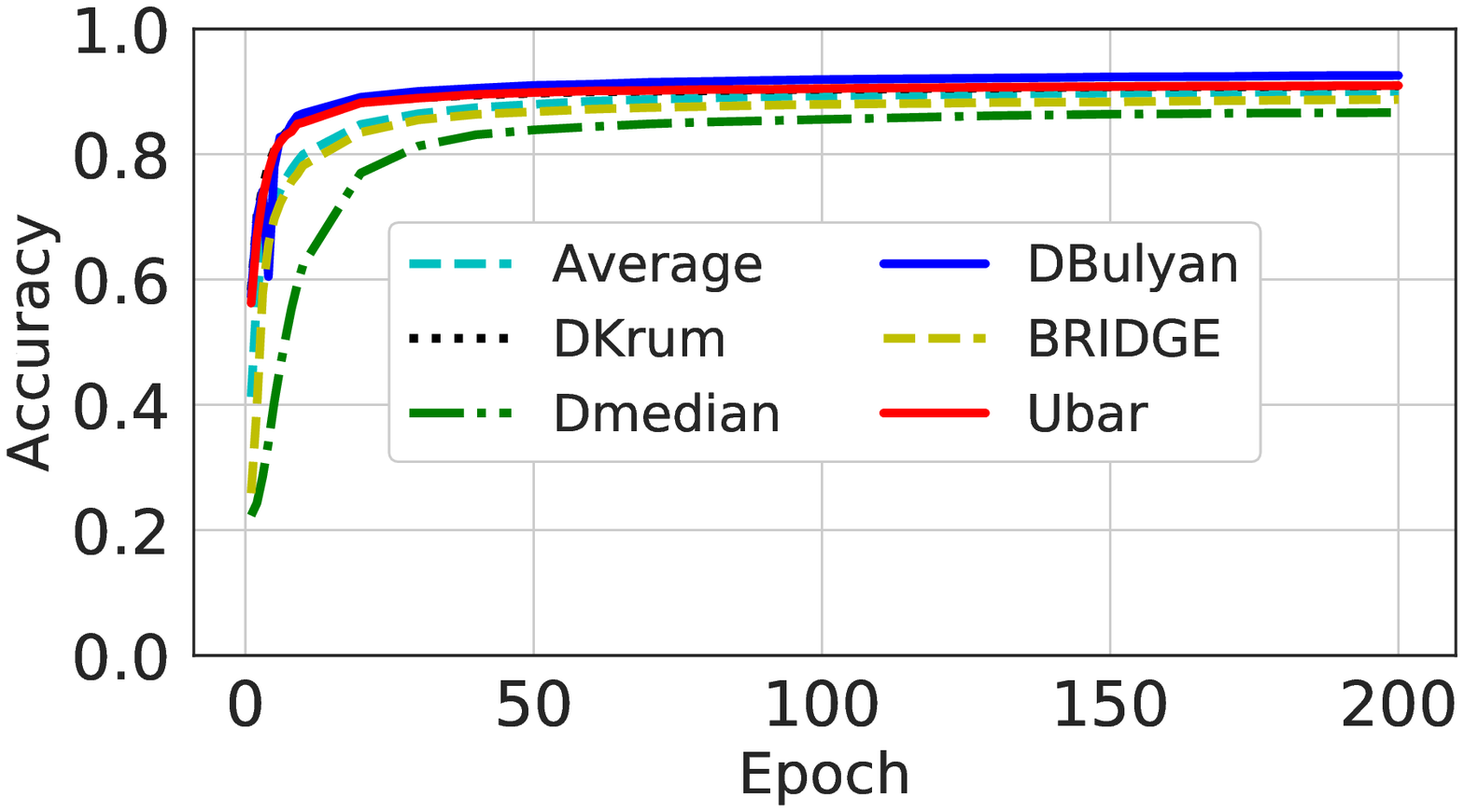}&\includegraphics[width=0.33\textwidth]{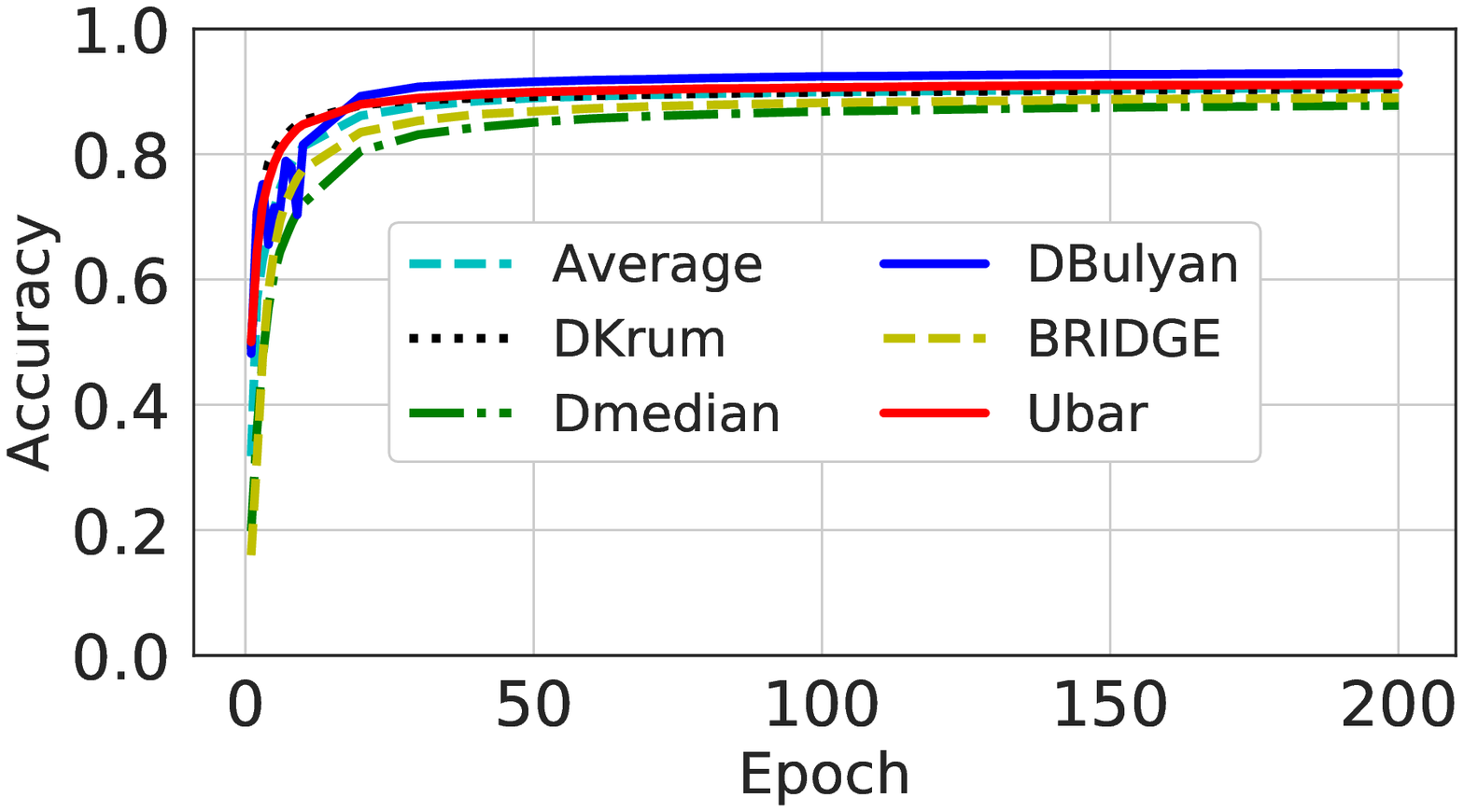}&\includegraphics[width=0.33\textwidth]{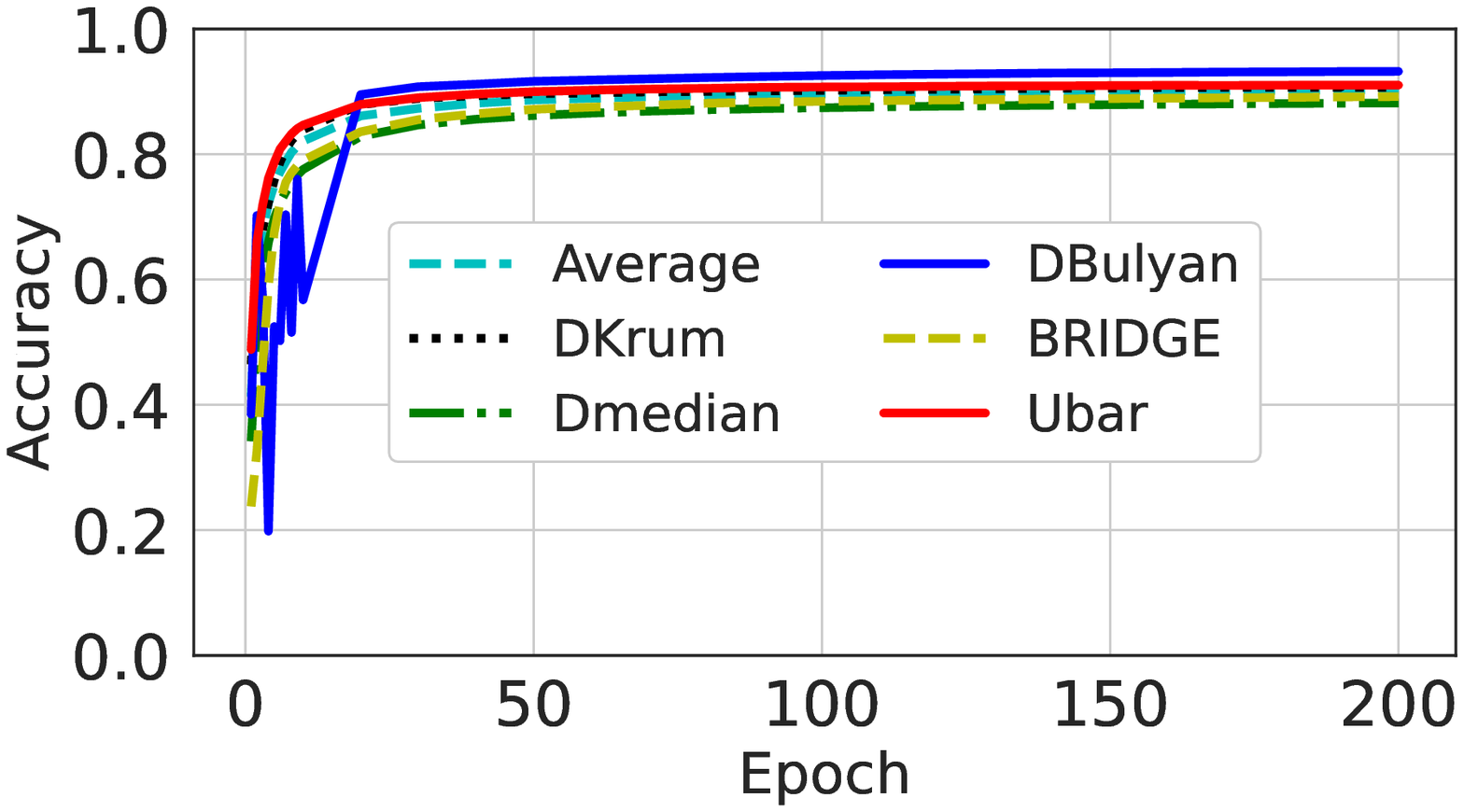}\\
			\rotatebox{90}{$\ \ \ \ \ \ \ \ \  $CIFAR10}&\includegraphics[width=0.33\textwidth]{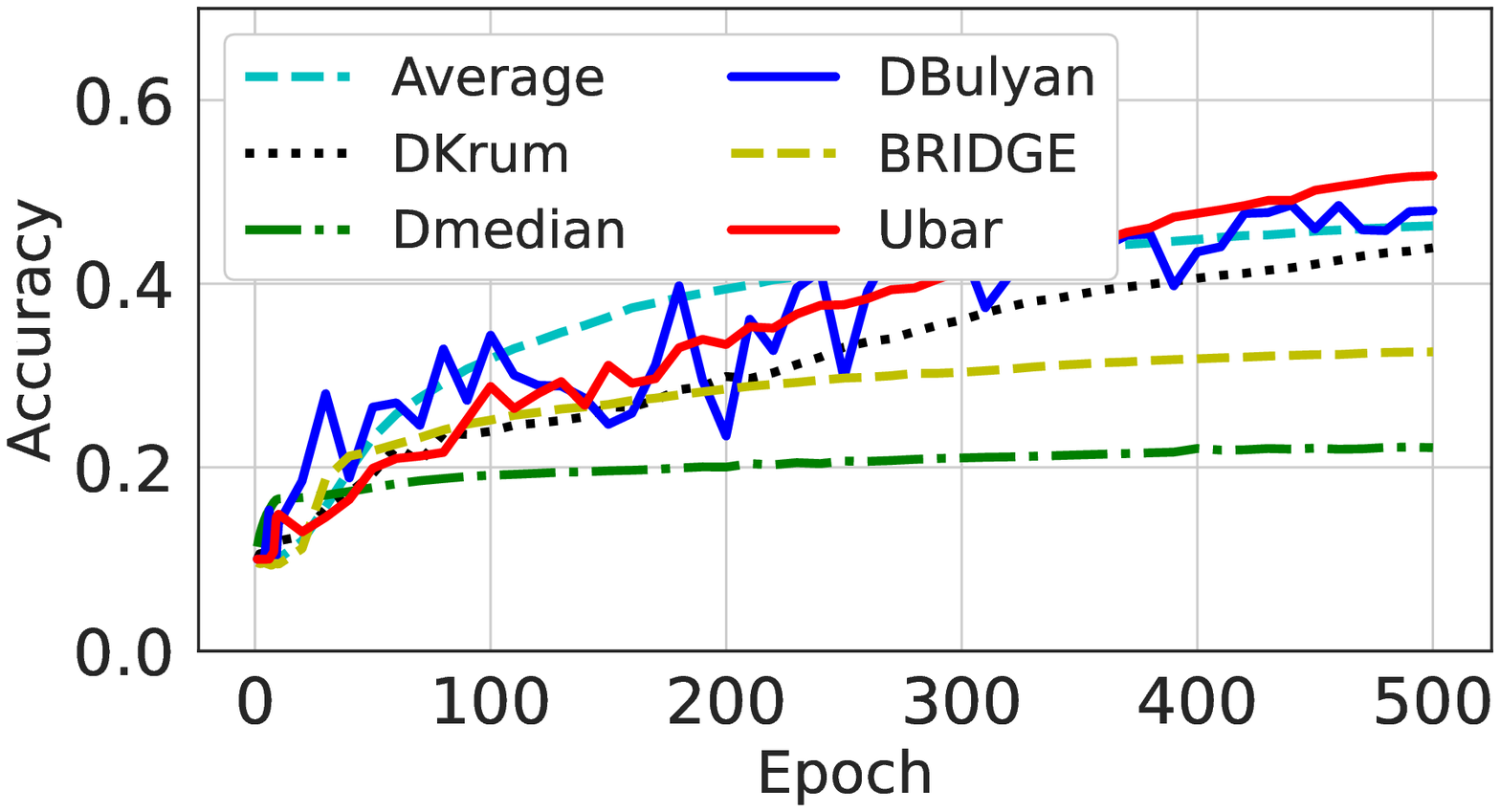}&\includegraphics[width=0.33\textwidth]{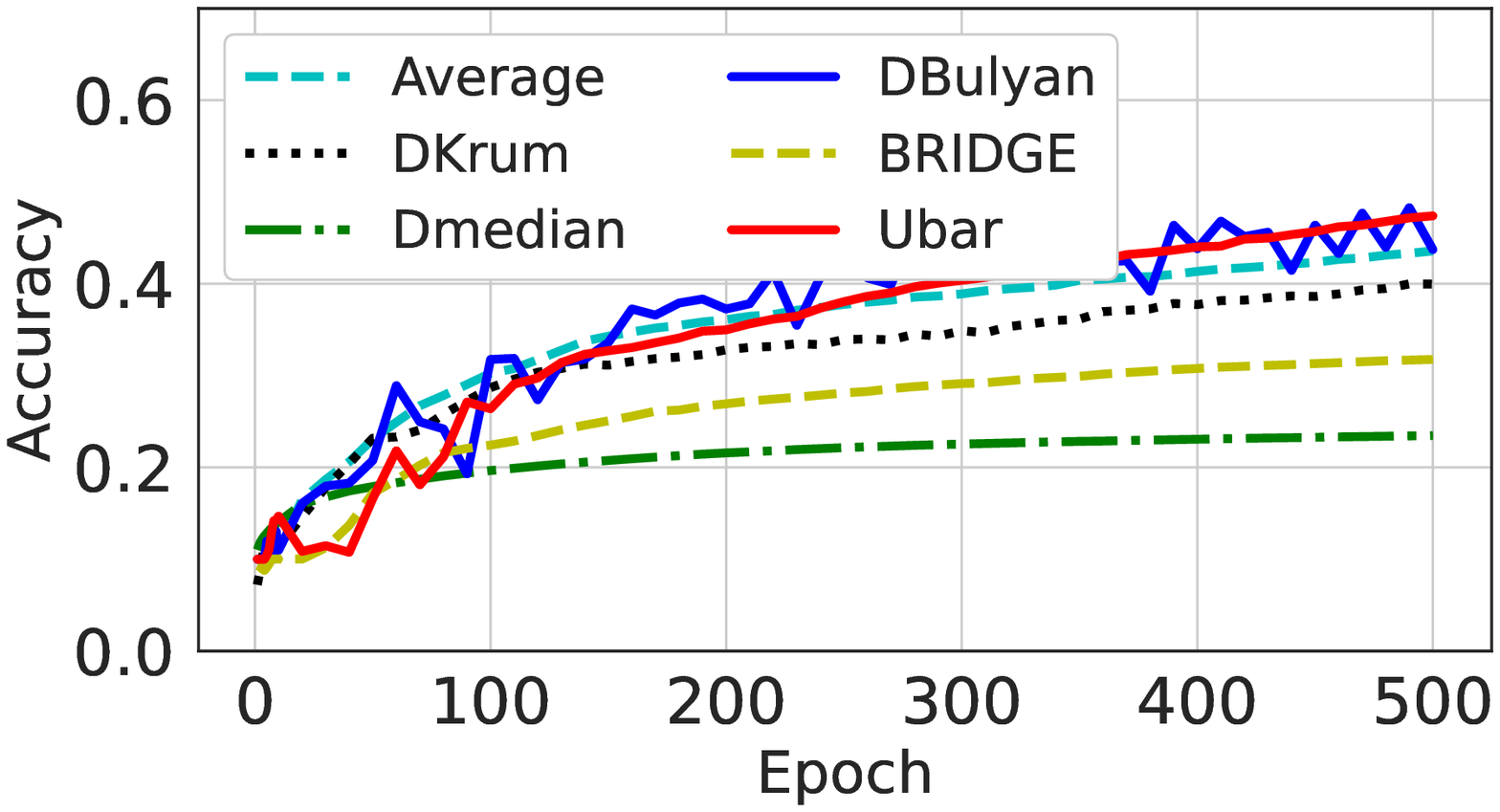}&\includegraphics[width=0.33\textwidth]{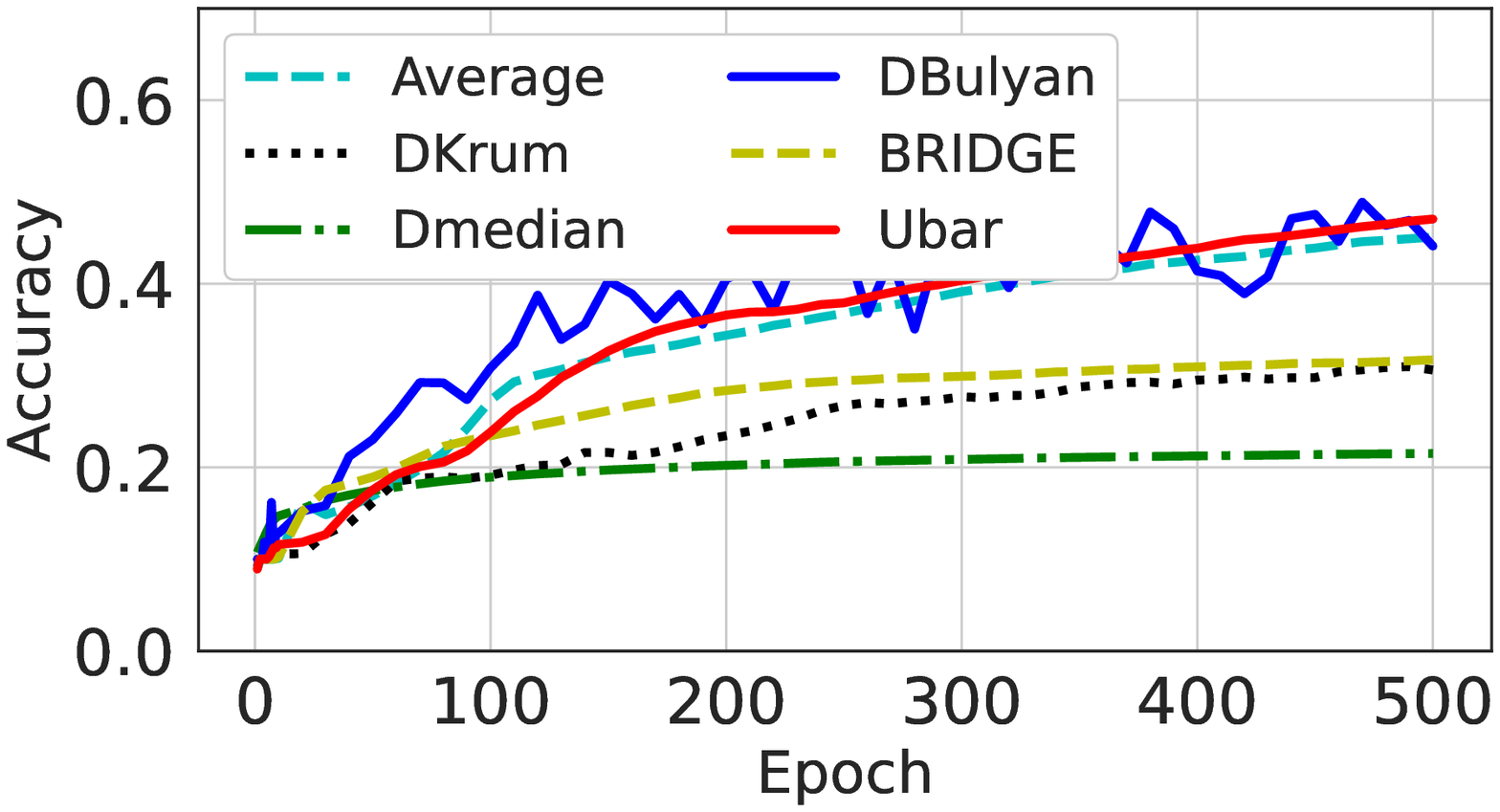}\\
			& (a) 30 nodes & (b) 50 nodes& (c) 70 nodes

	\end{tabular}}
	\captionof{figure}{The worst accuracy of the benign nodes with different network sizes.}\label{fig:cifar_vary_nodes}
\end{table*}

\begin{table*}[t!]\centering
	\resizebox{\textwidth}{!}{
		\begin{tabular}{c@{\hskip3pt}c@{}c@{}c}
			\rotatebox{90}{$\ \ \ \ \ \ \ \ \  $MNIST}&\includegraphics[width=0.33\textwidth]{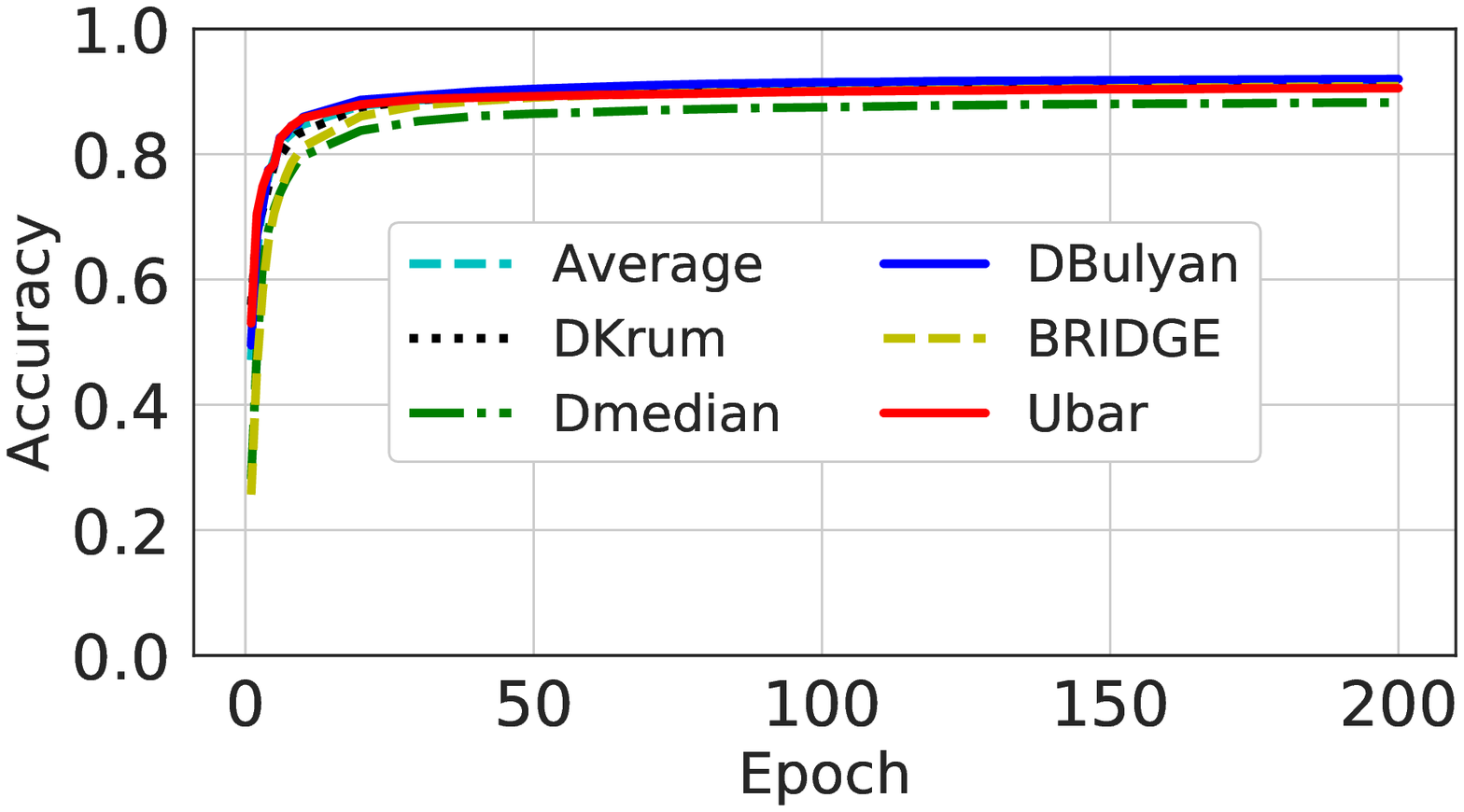}&\includegraphics[width=0.33\textwidth]{figs/mnist_accs30_200_0.4_0.0_hidden.eps}&\includegraphics[width=0.33\textwidth]{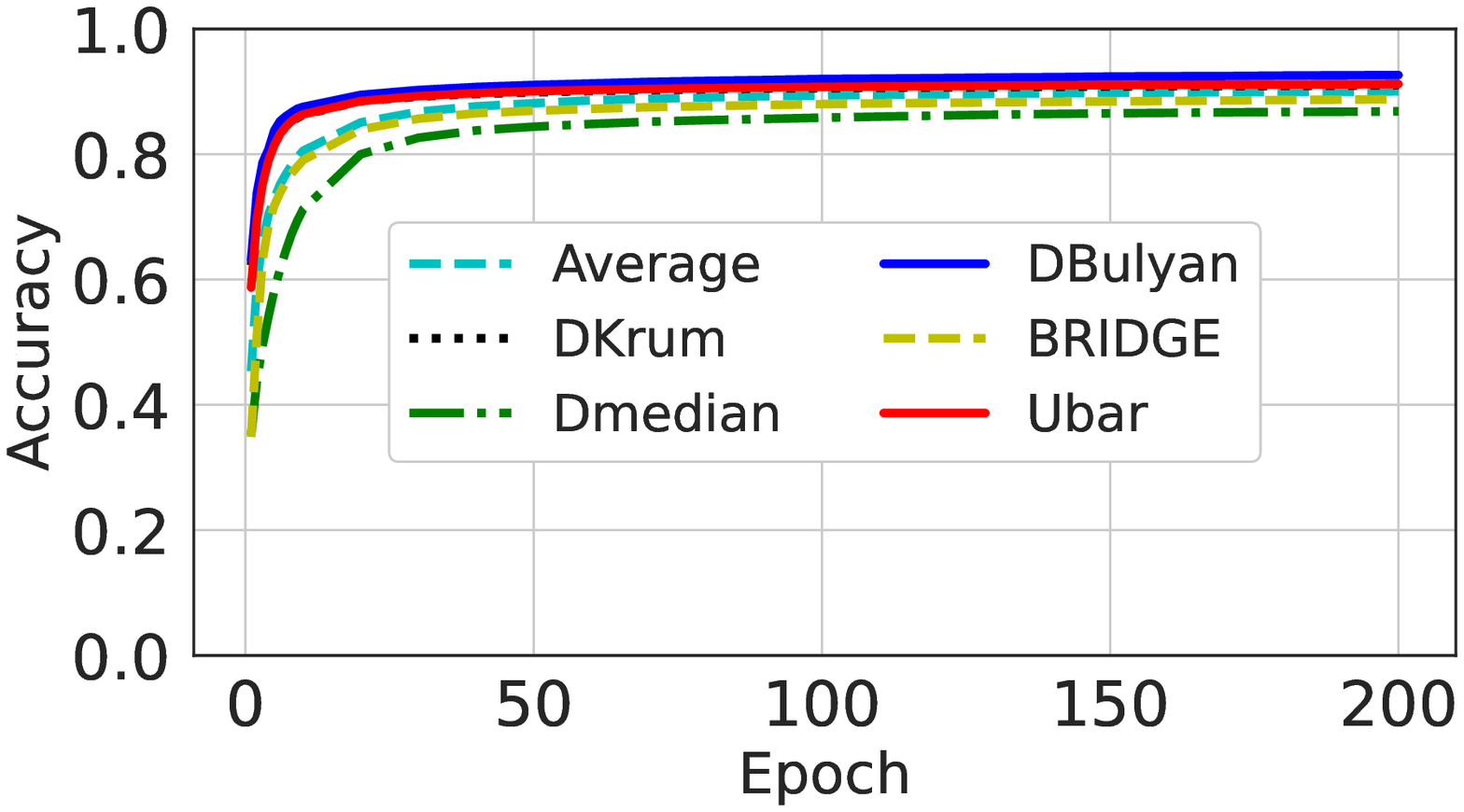}\\
			\rotatebox{90}{$\ \ \ \ \ \ \ \ \  $CIFAR10}&\includegraphics[width=0.33\textwidth]{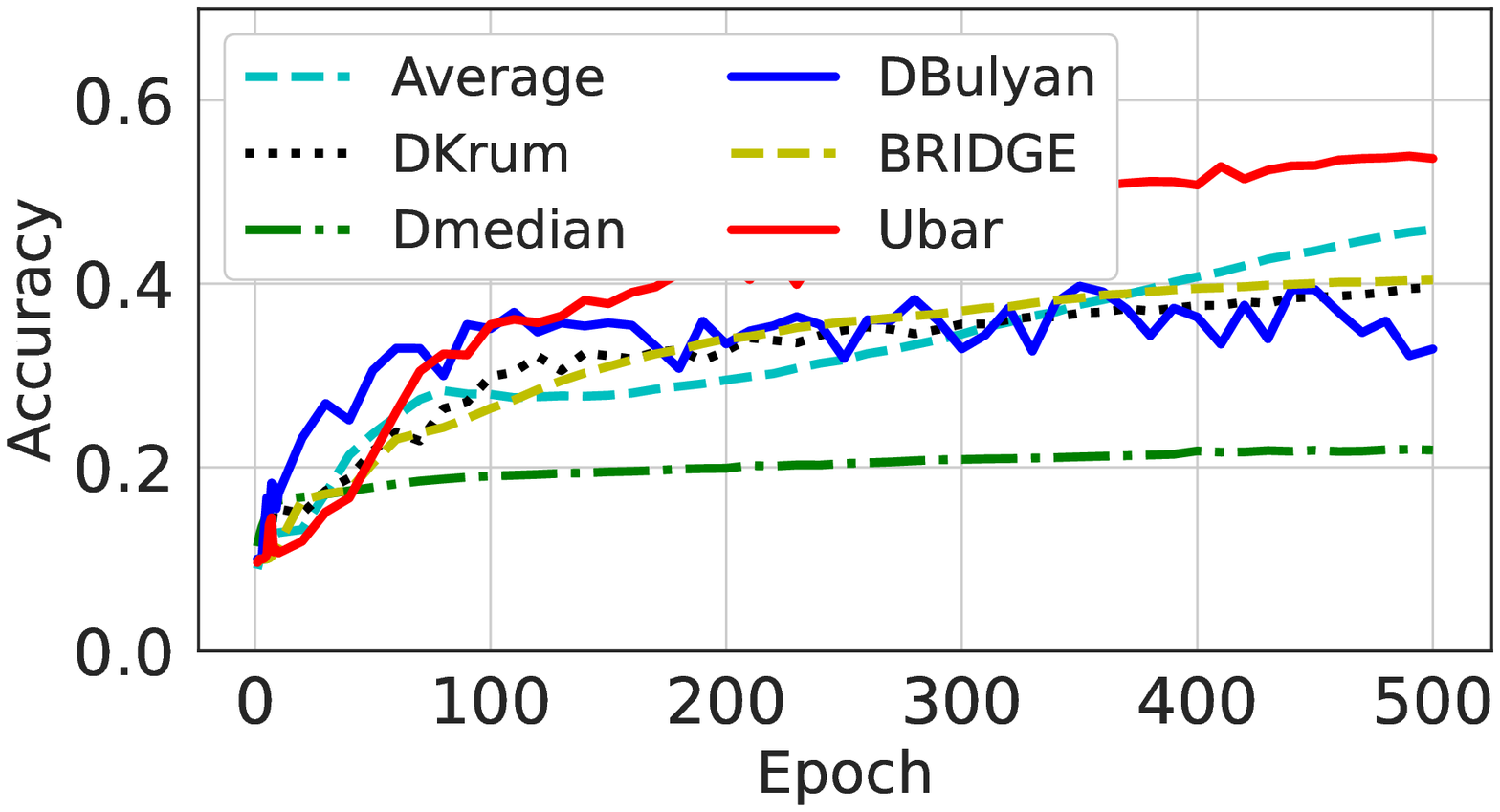}&\includegraphics[width=0.33\textwidth]{figs/cifar10_accs30_500_0.4_0.0_hidden.eps}&\includegraphics[width=0.33\textwidth]{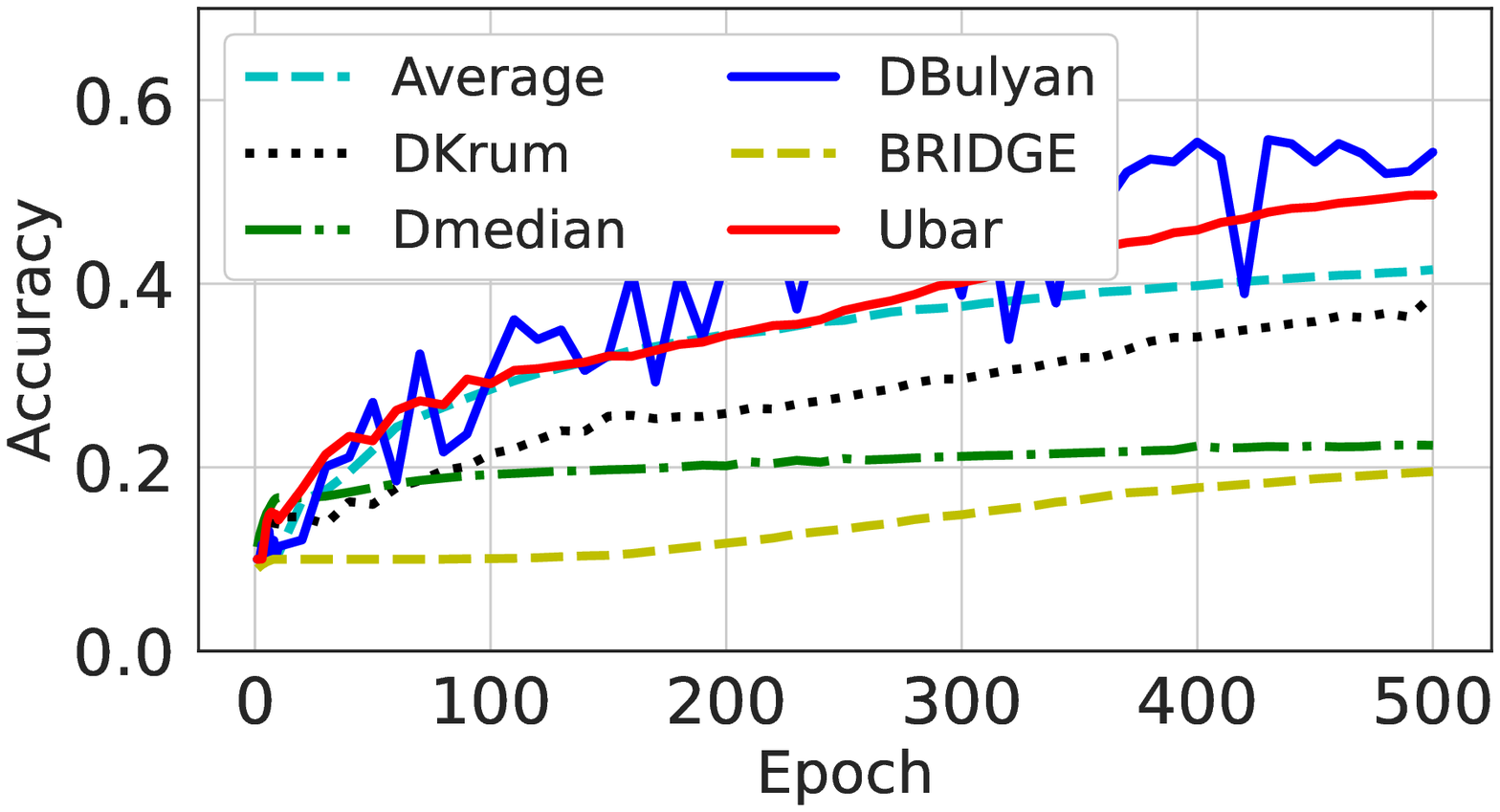}\\
			& (a) Connection ratio = 0.2 & (b) Connection ratio = 0.4& (c) Connection ratio = 0.6
	\end{tabular}}
	\captionof{figure}{The worst accuracy of the benign nodes with different connection ratios.}\label{fig:cifar_vary_connection}
\end{table*}

\section{Experiments}\label{sec:experiments}
\subsection{Experimental Setup and Configurations}
\bheading{Datasets.} We evaluate our defense solution with a DL-based image classification task. Specifically, we train a Convolutional Neural Network (CNN) over the  MNIST and CIFAR10 datasets \cite{krizhevsky2009learning}. This CNN includes two max–pooling layers and three fully connected layers \cite{mhamdi2018hidden}. We adopt a batch size of 256 and a fading learning rate $\lambda(k) = \lambda_0\frac{20}{20+k}$ where $k$ is the number of epochs and the initial learning rate $\lambda_0 = 0.05$. 

\bheading{Implementation of decentralized systems.} The network topology of the decentralized system in our consideration is defined by a connection rate between the nodes. A connection rate is the probability that a node is connected to another node. To ensure the connectivity assumption, we first generate a decentralized network in which all the nodes strictly follow the learning procedure. Then we randomly add Byzantine nodes to the network. To simulate various adversarial environments, we adopt a new parameter, Byzantine ratio, which is defined as the number of Byzantine nodes divided by the number of all nodes in the network. We assume that the Byzantine ratio of node $i$ is lower than $1- \rho_i$. Without lose of generality, we set $\rho_i$ as 0.4 for all benign nodes. We train the deep learning model in a synchronous mode. We simulate the operations of the decentralized system by running the nodes serially at each iteration. All our experiments are conducted on a server equipped with Xeon Silver 4214 @2.20 GHz CPUs and a NVIDIA Tesla P40 GPU.

\bheading{Baselines.}
Since there are very few works focusing on decentralized learning systems, we can extend existing aggregation rules from centralized systems to the
decentralized case as our baselines, since the estimates and gradients have the same dimensionality and network
structure. We consider three popular Byzantine-resilient solutions: Krum
\cite{blanchard2017machine}, marginal median \cite{xie2018generalized,yin2018byzantine}, and Bulyan
\cite{mhamdi2018hidden}. 

It is worth noting that in a centralized distributed system, the maximal number of
Byzantine workder nodes connected to the parameter server is usually assumed. This does not
hold in a decentralized system, as the number of Byzantine nodes connected to each benign node
varies greatly. To meet this assumption, we approximately
calculate the number of Byzantine nodes allowed for each benign node in Equation
\ref{eq:num-fault}. In this equation, $\hat{n}_{i}$ is the number of Byzantine nodes connected to node $i$, $\rho_{central}$ is the maximal ratio of Byzantine nodes allowed in a centralized distributed defense, and $\lceil \cdot \rceil$ is the ceil function.

\begin{equation}
\label{eq:num-fault}
   \hat{n}_{i} =\lceil|\mathcal{N}_i|\cdot\min\{1-\rho_i, \rho_{central}\}\rceil
\end{equation}

\bheading{DKrum.} Similar to Krum, let $j \neq j'$ be two neighbors of node $i$ and we denote $j \rightarrow j'$ the $|\mathcal{N}_i| - \hat{n}_{i} + 2$ closest estimate vectors to the estimate of node $j$. Then, we calculate the score of each neighbor:

\begin{equation}\label{equ:krumscore}\nonumber
   s(j) = \sum_{j\rightarrow j'}||x_{j}- x_{j'}||^2
\end{equation}
We select the estimate with the minimal score as the aggregated estimate. Formally, the aggregated rule  DKrum is defined as
\begin{equation}\nonumber
    \mathcal{R}_{DKrum} = x_{j^*}, \ j^* = \argmin_{j\in \mathcal{N}_i}s(j).
\end{equation}

\bheading{Dmedian.} To apply the marginal median solution\cite{xie2018generalized,yin2018byzantine} to
decentralized systems, we only need to replace the gradients with the received
estimates. Specifically, the aggregated rule $Dmedian$ is defined as
\begin{equation}
   \mathcal{R}_{Dmedian} = \texttt{MarMed}\{x_{j}, j\in \mathcal{N}_i\}
\end{equation}
where $\texttt{MarMed}$ is the marginal median function defined in \cite{xie2018generalized,yin2018byzantine}. Informally, the $m$-th dimensional value of $\mathcal{R}_{Dmedian}$ is the median of the $m$-th dimensional elements of all estimates in $\mathcal{N}_i$.

\bheading{DBulyan.} At each iteration, node $i$ first recursively uses DKrum to select $|\mathcal{N}_i| - 2\hat{n}_{i}$ estimates, i.e., $\{x_j, j \in \mathcal{N}_i^{DKrum}\}$, where $|\mathcal{N}_i^{DKrum}| = |\mathcal{N}_i| - 2\hat{n}_{i}$. Specifically, node $i$ uses DKrum to select one estimate from its neighbors and deletes the corresponding node from its neighbors. Then, node $i$ recursively selects the remaining estimates using DKrum. Finally, it adopts a median-based method to aggregate the estimates in $\{x_j, j \in \mathcal{N}_i^{DKrum}\}$. Formally, the $m$-th coordinate of the aggregated estimate is calculated as
\begin{equation}
   \mathcal{R}_{DBulyan}[m] = \frac{1}{\beta}\sum_{j \in \mathcal{M}[m]}x_{j}[m]
\end{equation}
where $\beta =|\mathcal{N}_i| - 4\hat{n}_{i}$ and $\mathcal{M}[m]$ is the set of neighbors with the size of $\beta$. The sum of the $m$-th elements to its median is minimal among all subsets of $\mathcal{N}_i^{DKrum}$ with the size $\beta$.

In addition to the above defenses from centralized systems, we also implement BRIDGE \cite{yang2019bridge} for comparison, which is designed specifically for decentralized systems, i.e., BRIDGE \cite{yang2019bridge}.
For all these defenses, we set $\alpha$ to be 0.5 in the GUF. For baseline, we consider the same decentralized system configuration without Byzantine nodes, and using the Average Aggregation rule (Equation \ref{equ:original_ar}). The model trained from this setting can be regarded as the optimal one.

\bheading{Performance Metric.} For each defense deployed in the decentralized system, we measure the testing
accuracy of the trained model on each benign node, and report the worst accuracy among all nodes to
represent the effectiveness of this defense.

\begin{table*}[t]\centering
	\resizebox{\textwidth}{!}{
		\begin{tabular}{c@{\hskip3pt}c@{}c@{}c}
			\rotatebox{90}{$\ \ \ \ \ \ \ \ \  $MNIST}&\includegraphics[width=0.33\textwidth]{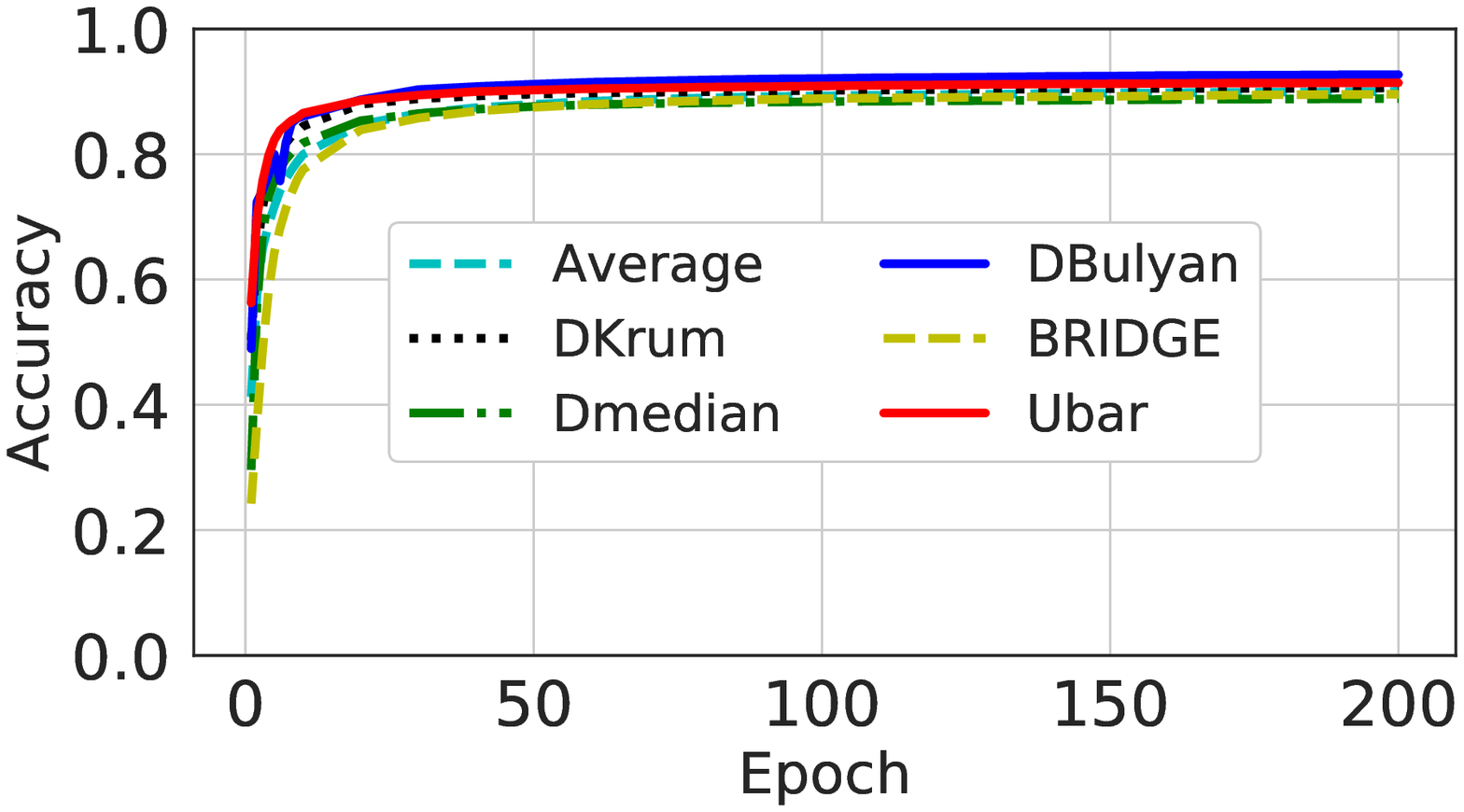}&\includegraphics[width=0.33\textwidth]{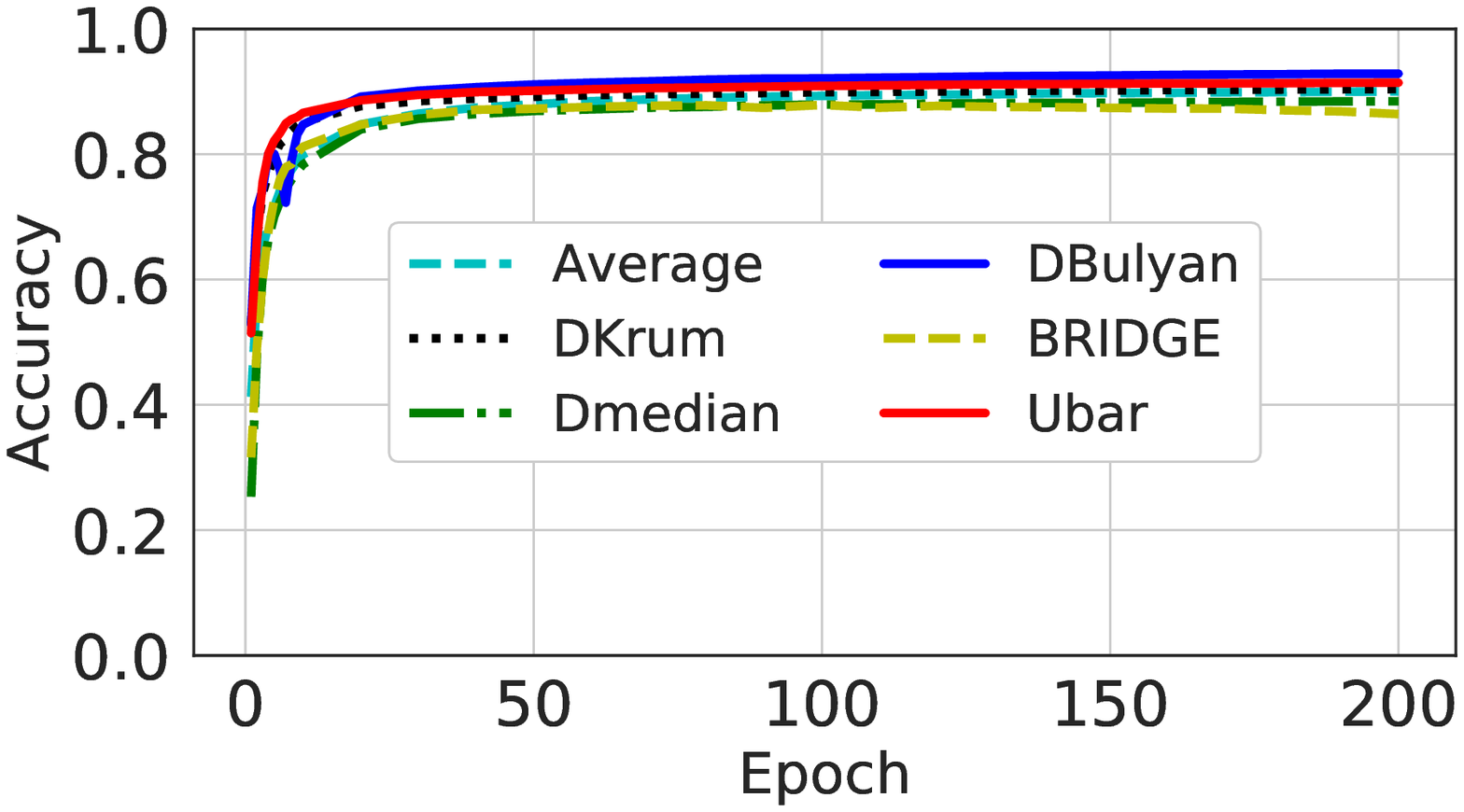}&\includegraphics[width=0.33\textwidth]{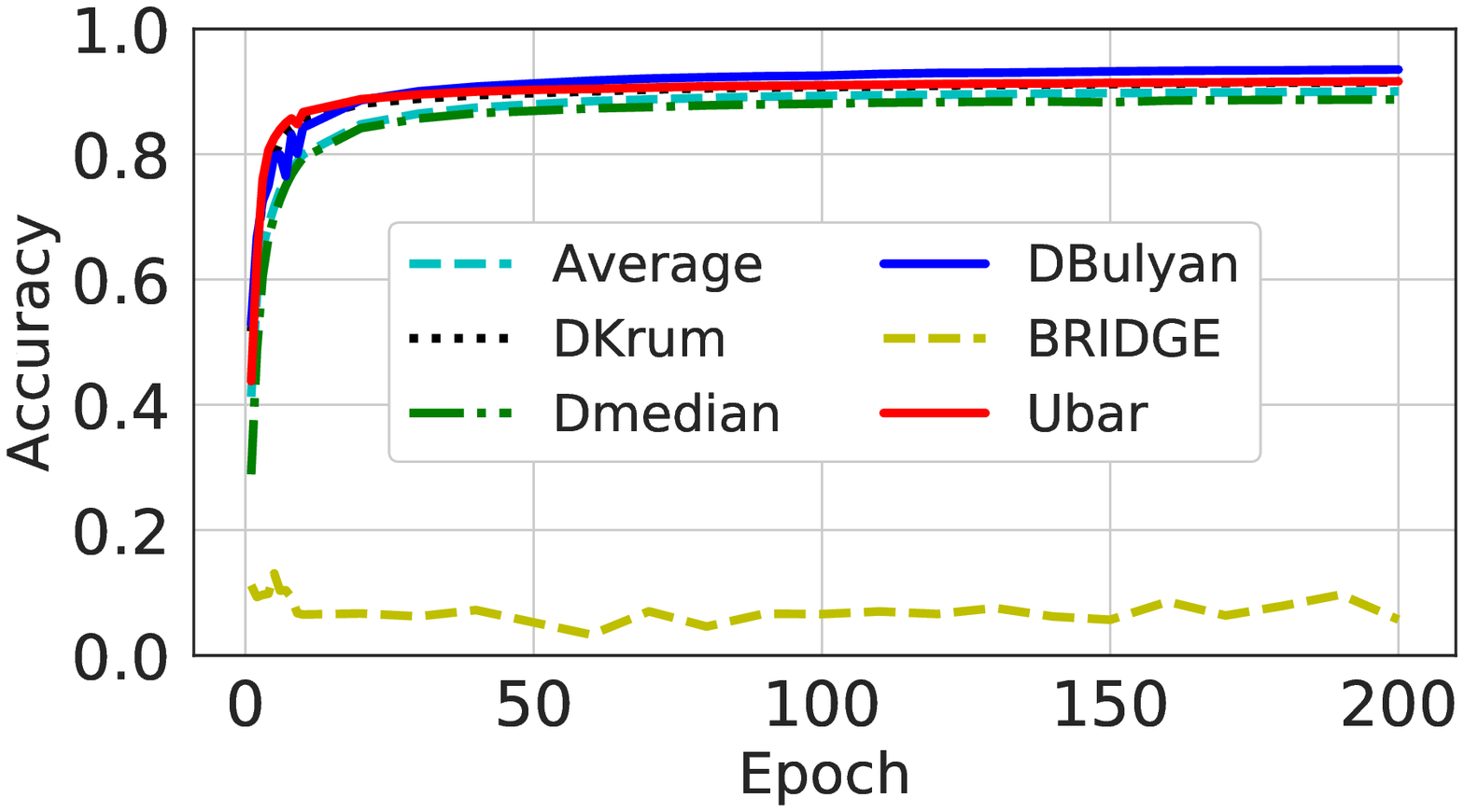}\\
			\rotatebox{90}{$\ \ \ \ \ \ \ \ \  $CIFAR10}&\includegraphics[width=0.33\textwidth]{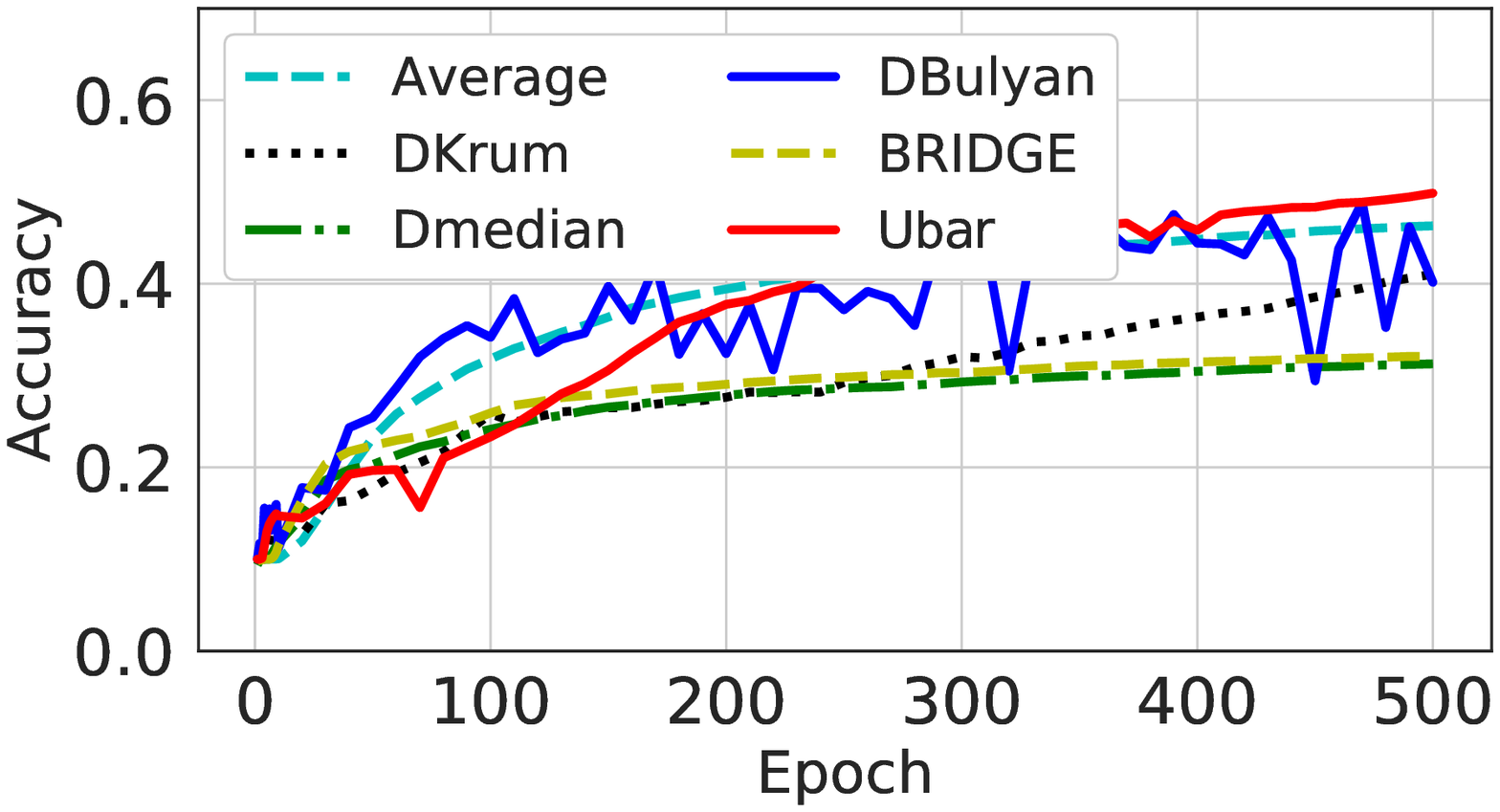}&\includegraphics[width=0.33\textwidth]{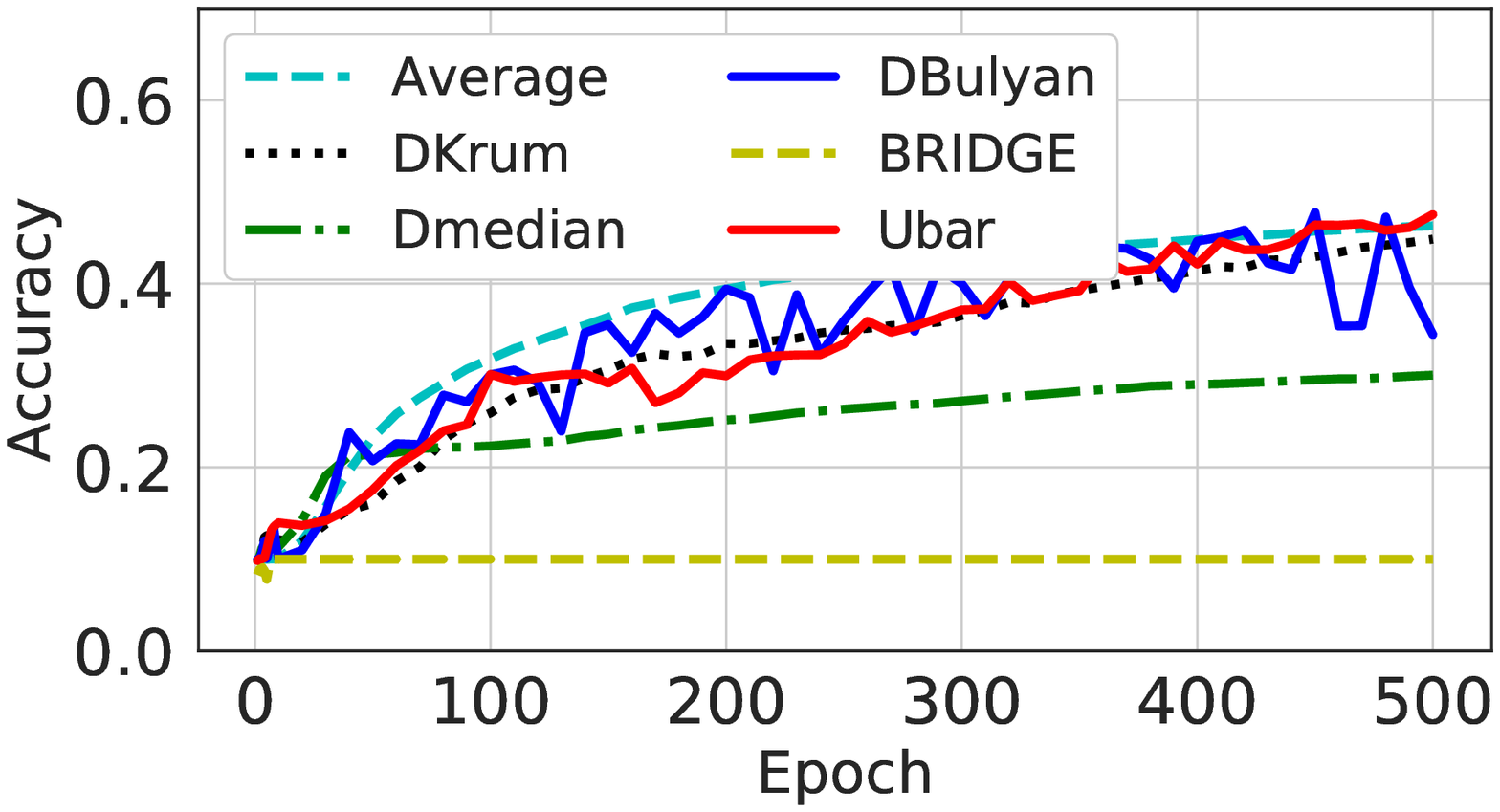}&\includegraphics[width=0.33\textwidth]{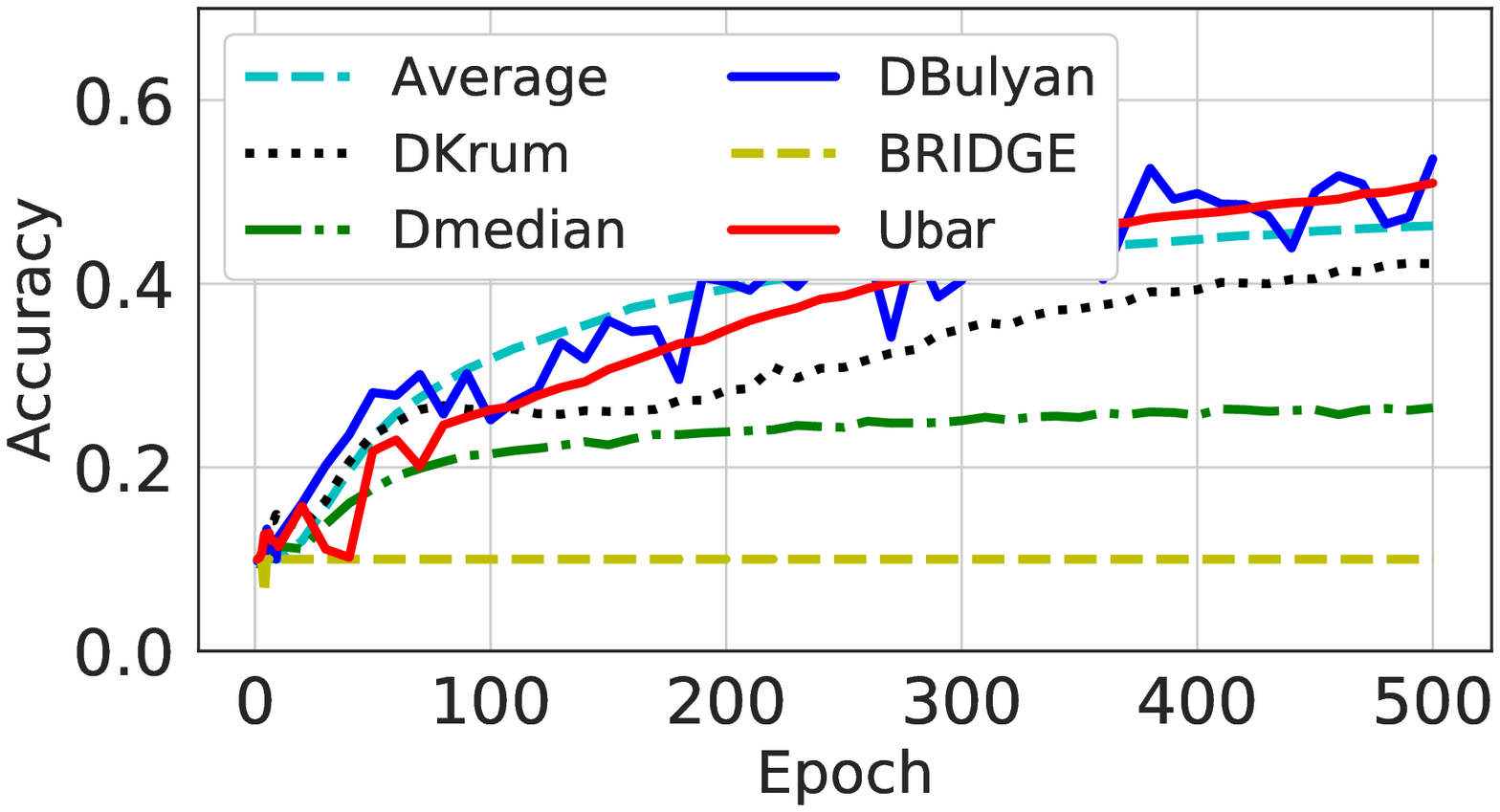}\\
			& (a) Byzantine ratio = 0.1 & (b) Byzantine ratio = 0.3& (c) Byzantine ratio = 0.5

	\end{tabular}}
	\captionof{figure}{The worst accuracy of the benign nodes under the Gaussian attack.}\label{fig:cifar_simple_attack}
\end{table*}

\begin{table*}[t]\centering
	\resizebox{\textwidth}{!}{
		\begin{tabular}{c@{\hskip3pt}c@{}c@{}c}
			\rotatebox{90}{$\ \ \ \ \ \ \ \ \  $MNIST}&\includegraphics[width=0.33\textwidth]{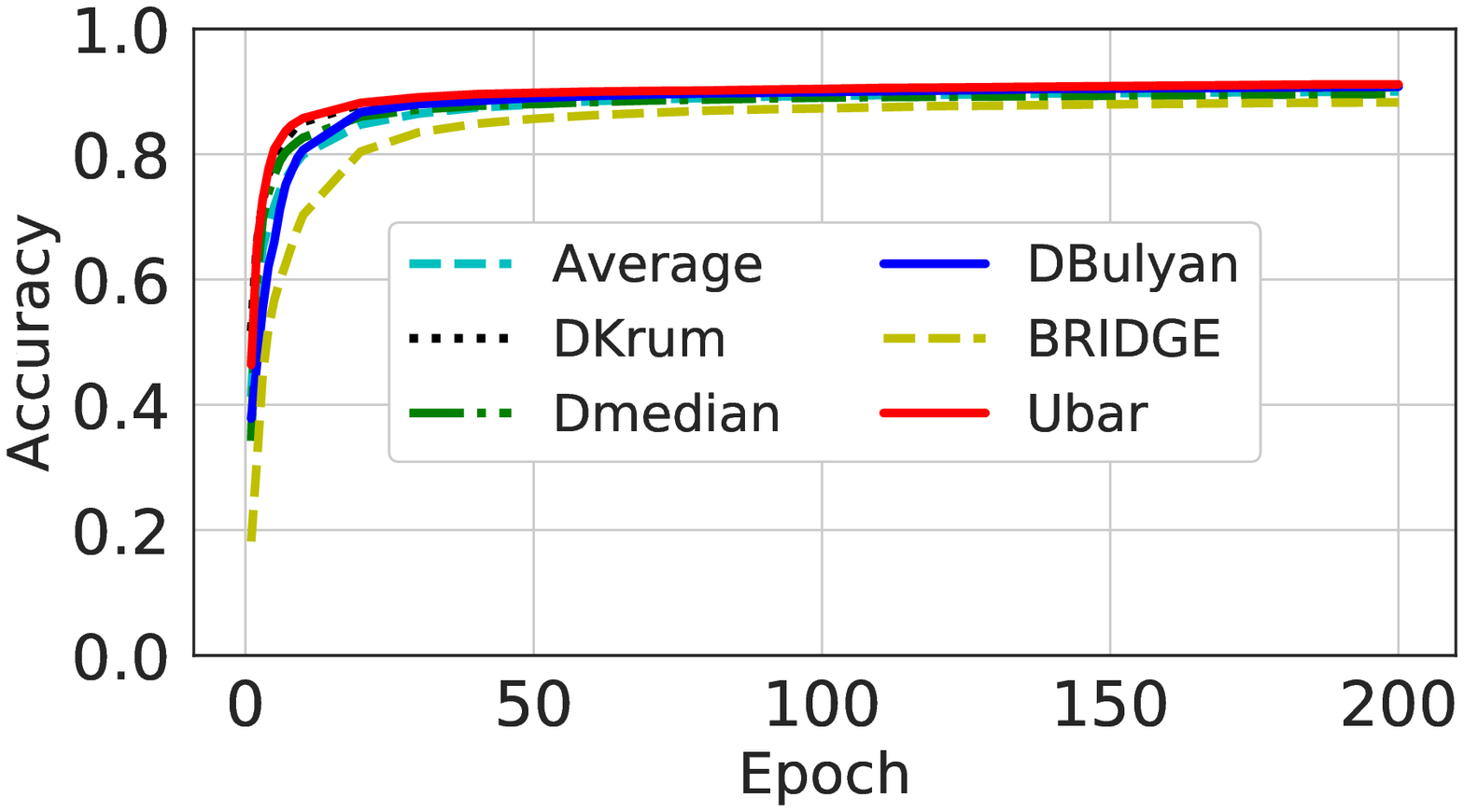}&\includegraphics[width=0.33\textwidth]{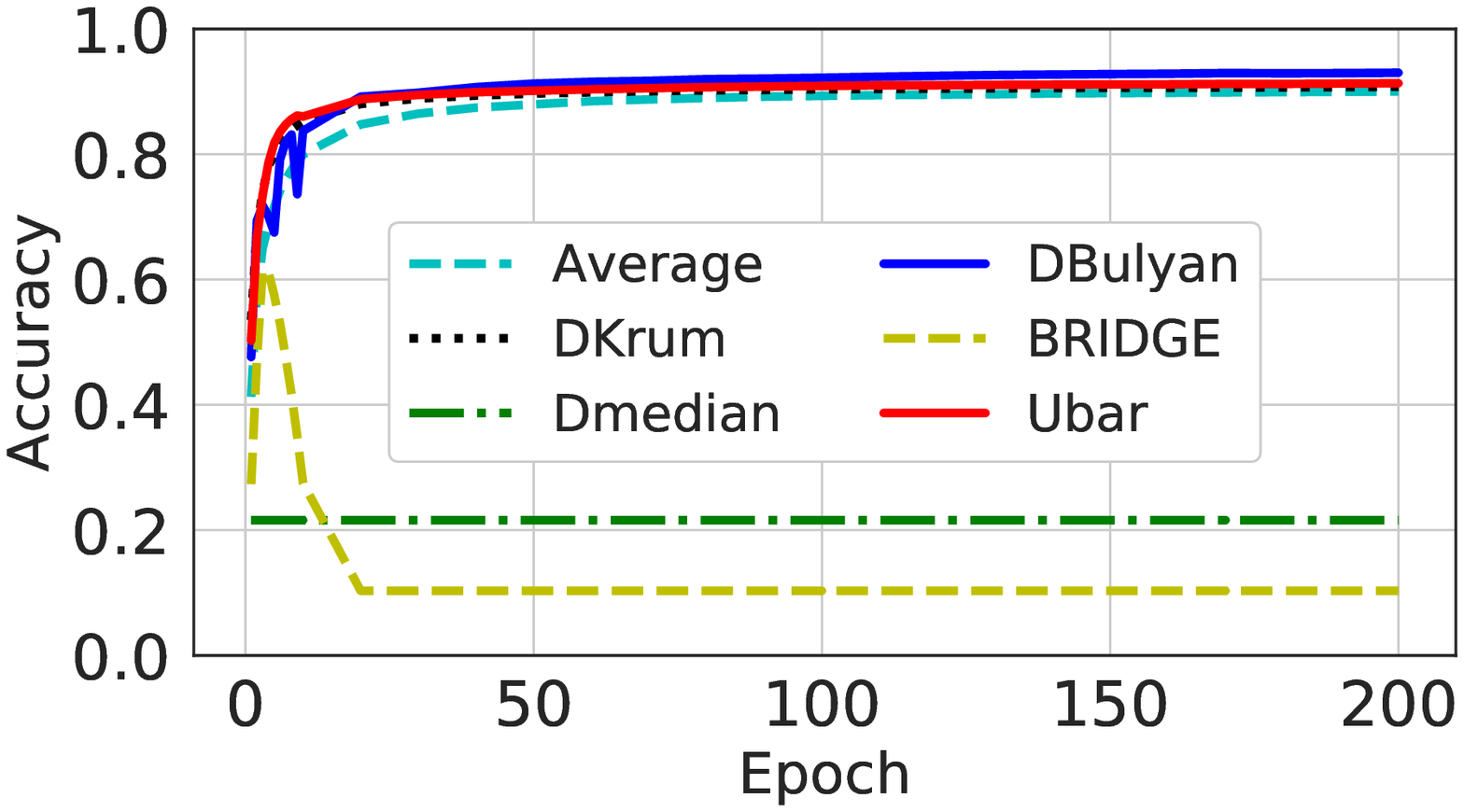}&\includegraphics[width=0.33\textwidth]{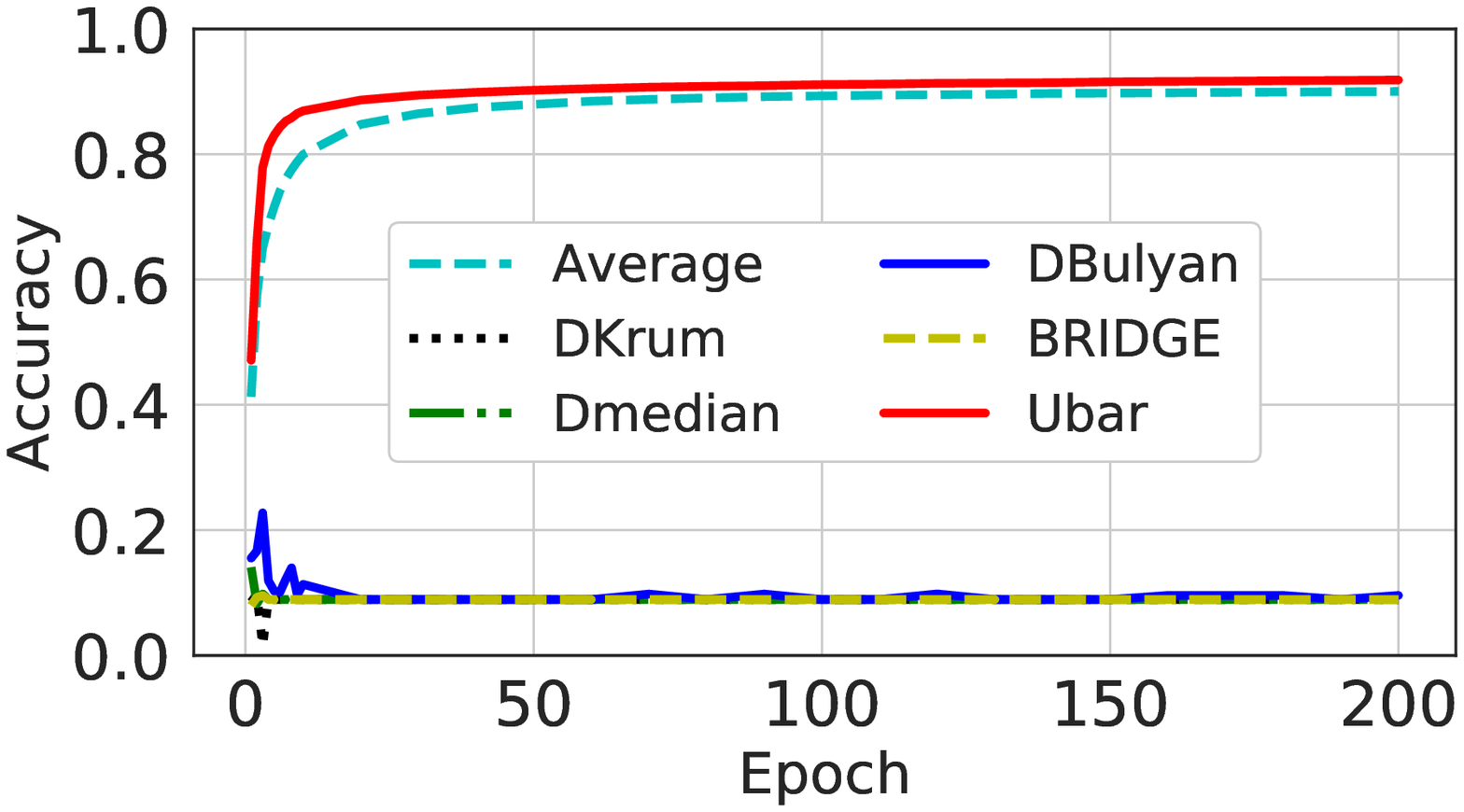}\\
			\rotatebox{90}{$\ \ \ \ \ \ \ \ \  $CIFAR10}&\includegraphics[width=0.33\textwidth]{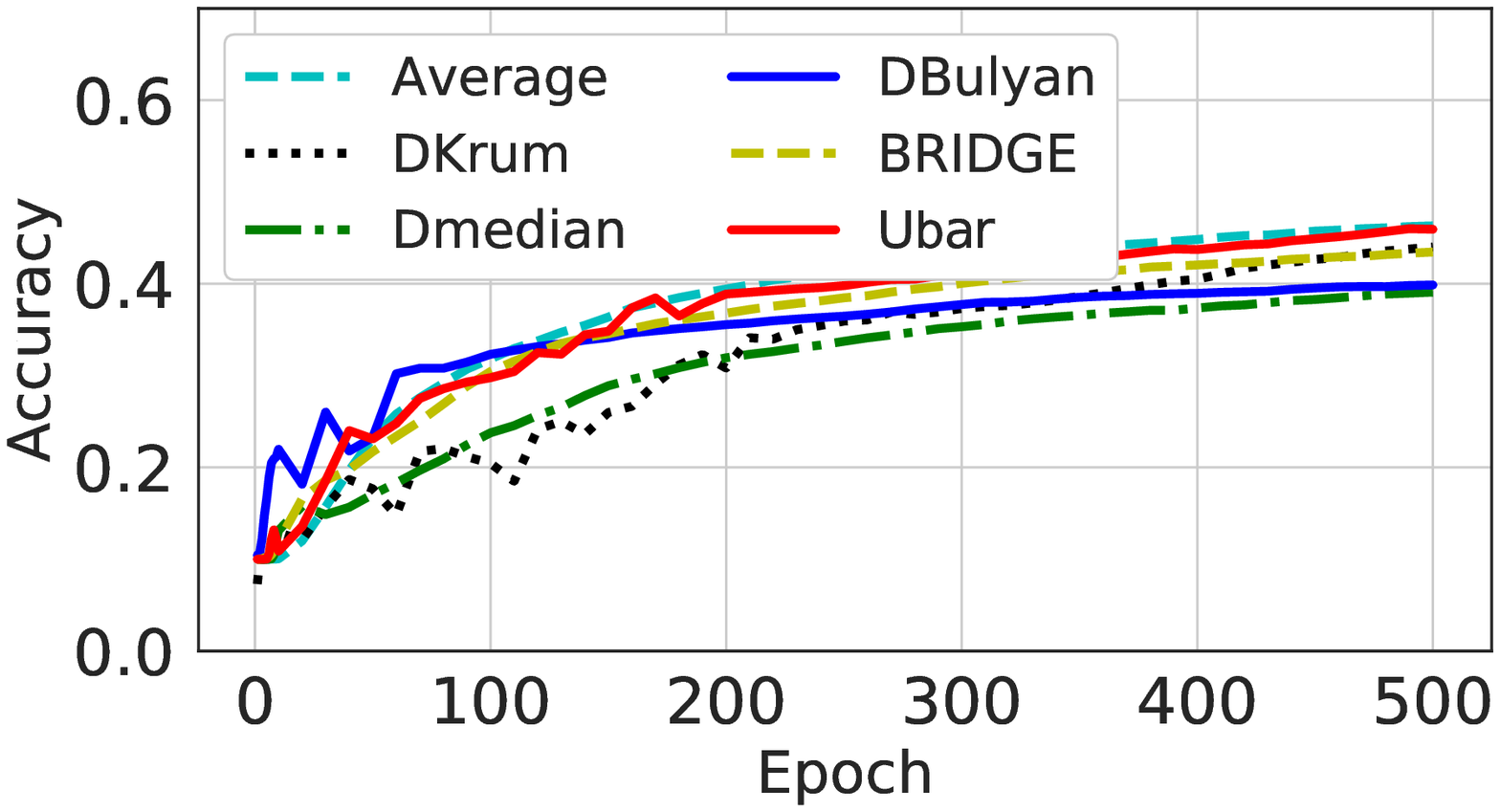}&\includegraphics[width=0.33\textwidth]{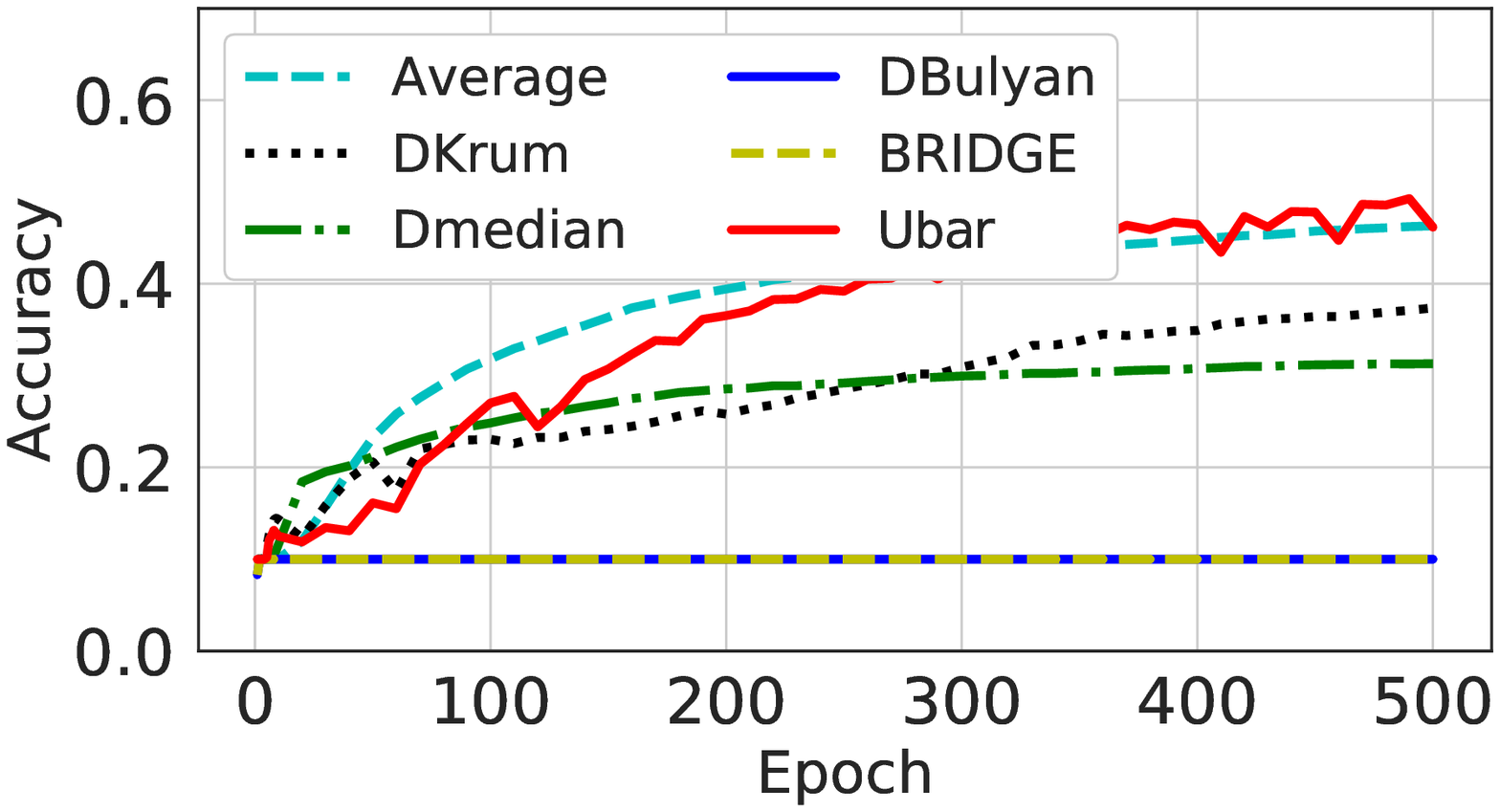}&\includegraphics[width=0.33\textwidth]{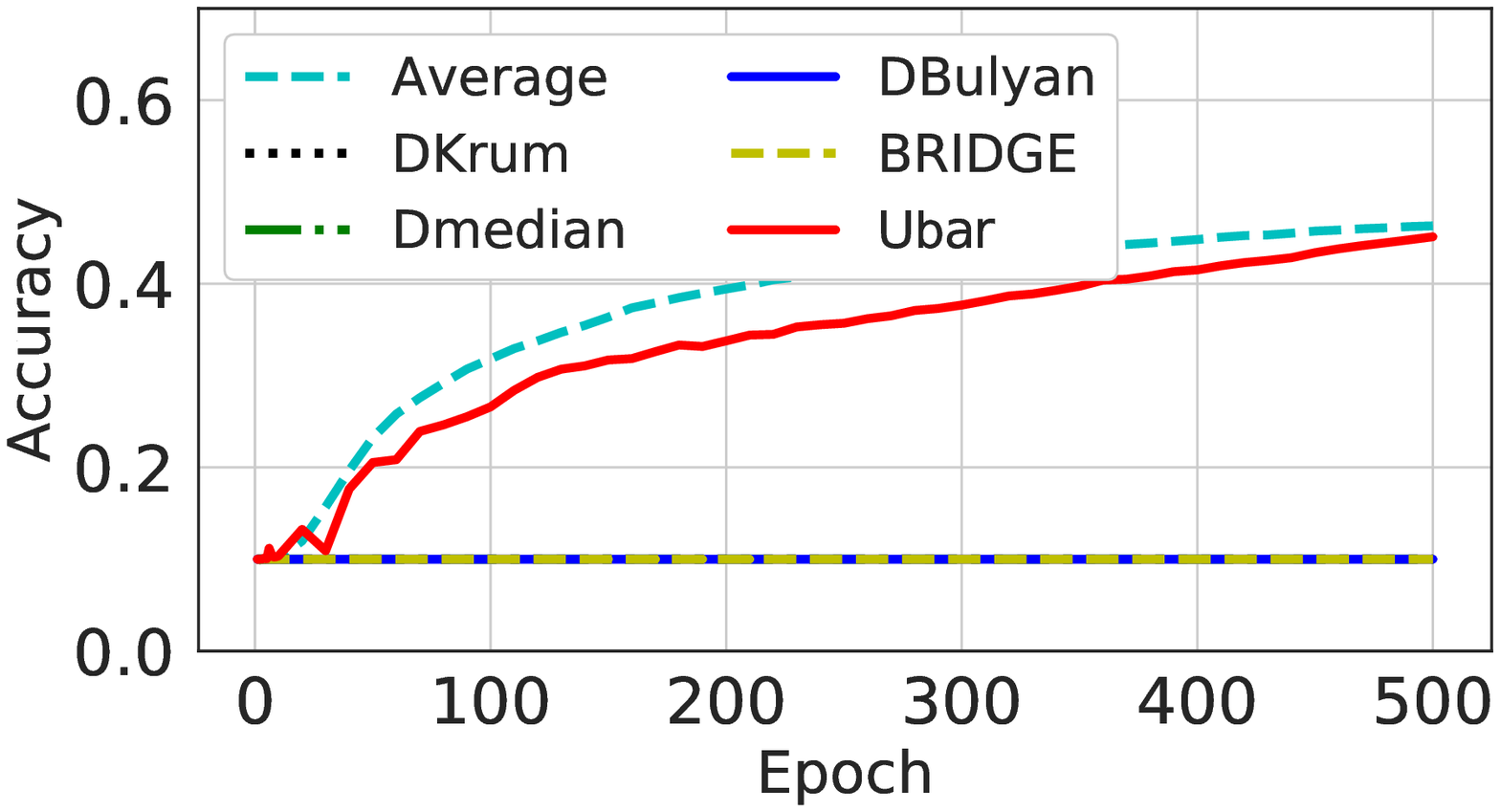}\\
			& (a) Byzantine ratio = 0.1 & (b) Byzantine ratio = 0.3& (c) Byzantine ratio = 0.5

	\end{tabular}}
	\captionof{figure}{The worst accuracy of the benign nodes under the bit-flip attack.}\label{fig:cifar_bitflip_attack}
\end{table*}

\begin{table*}[t]\centering
	\resizebox{\textwidth}{!}{
		\begin{tabular}{c@{\hskip3pt}c@{}c@{}c}
			\rotatebox{90}{$\ \ \ \ \ \ \ \ \  $MNIST}&\includegraphics[width=0.33\textwidth]{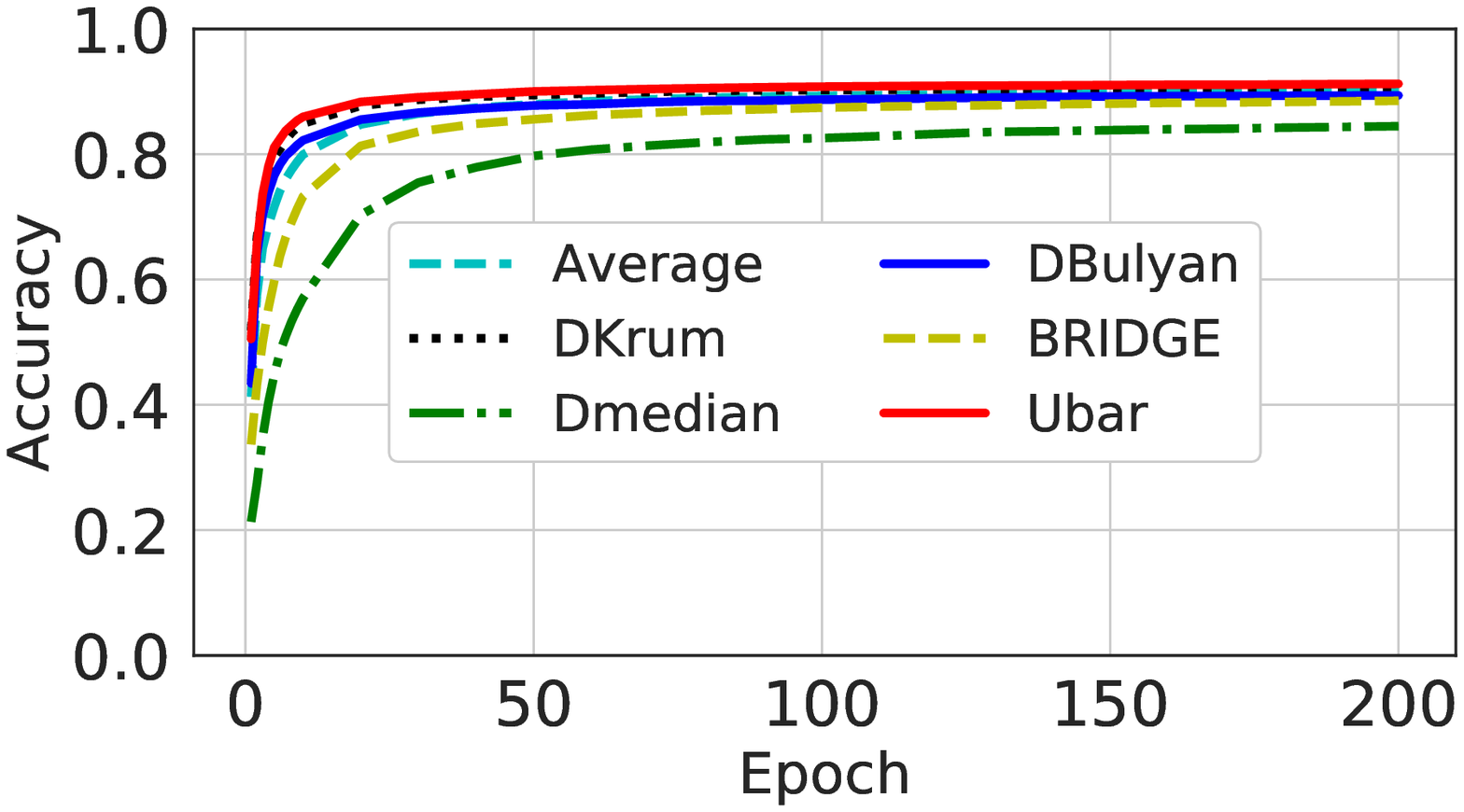}&\includegraphics[width=0.33\textwidth]{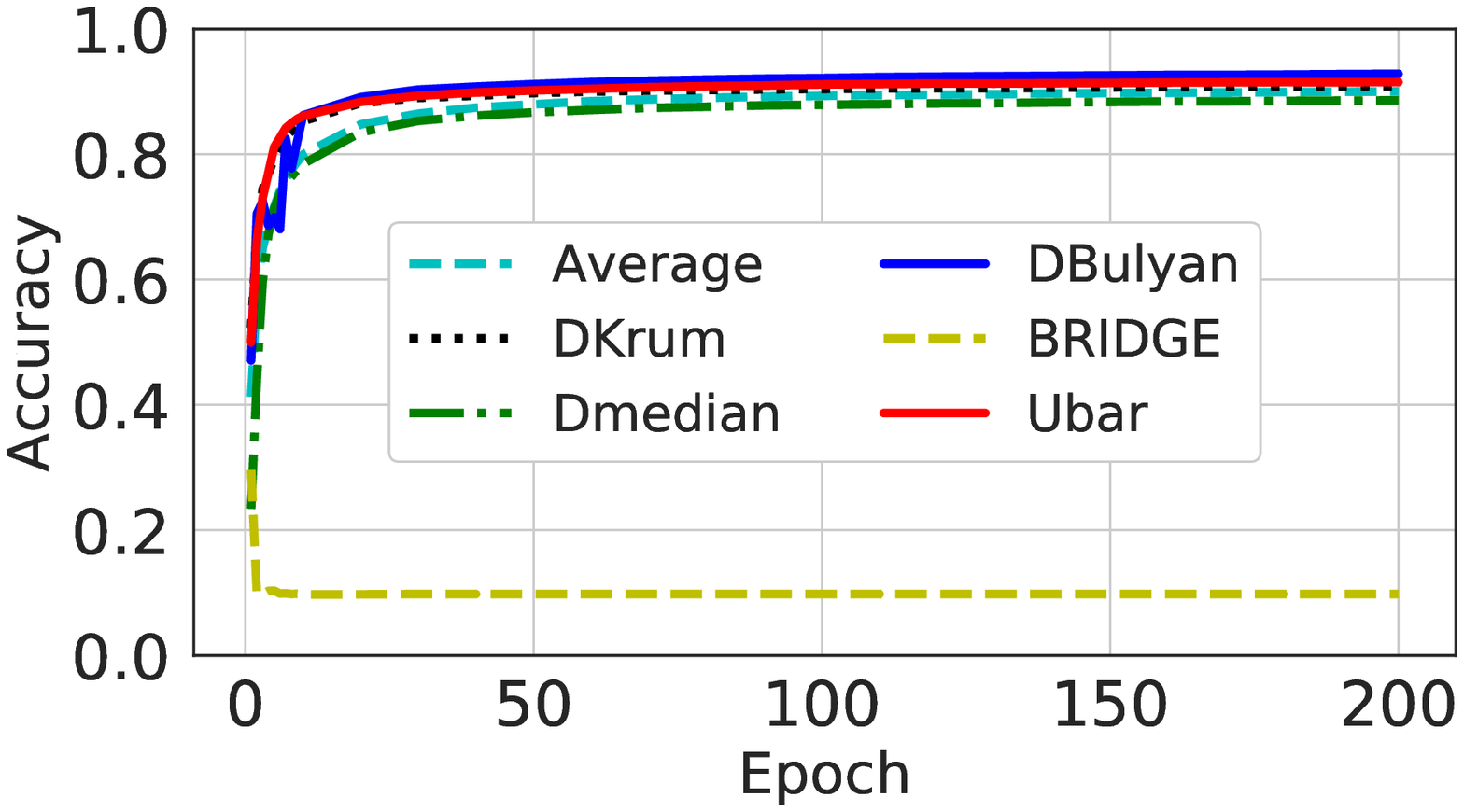}&\includegraphics[width=0.33\textwidth]{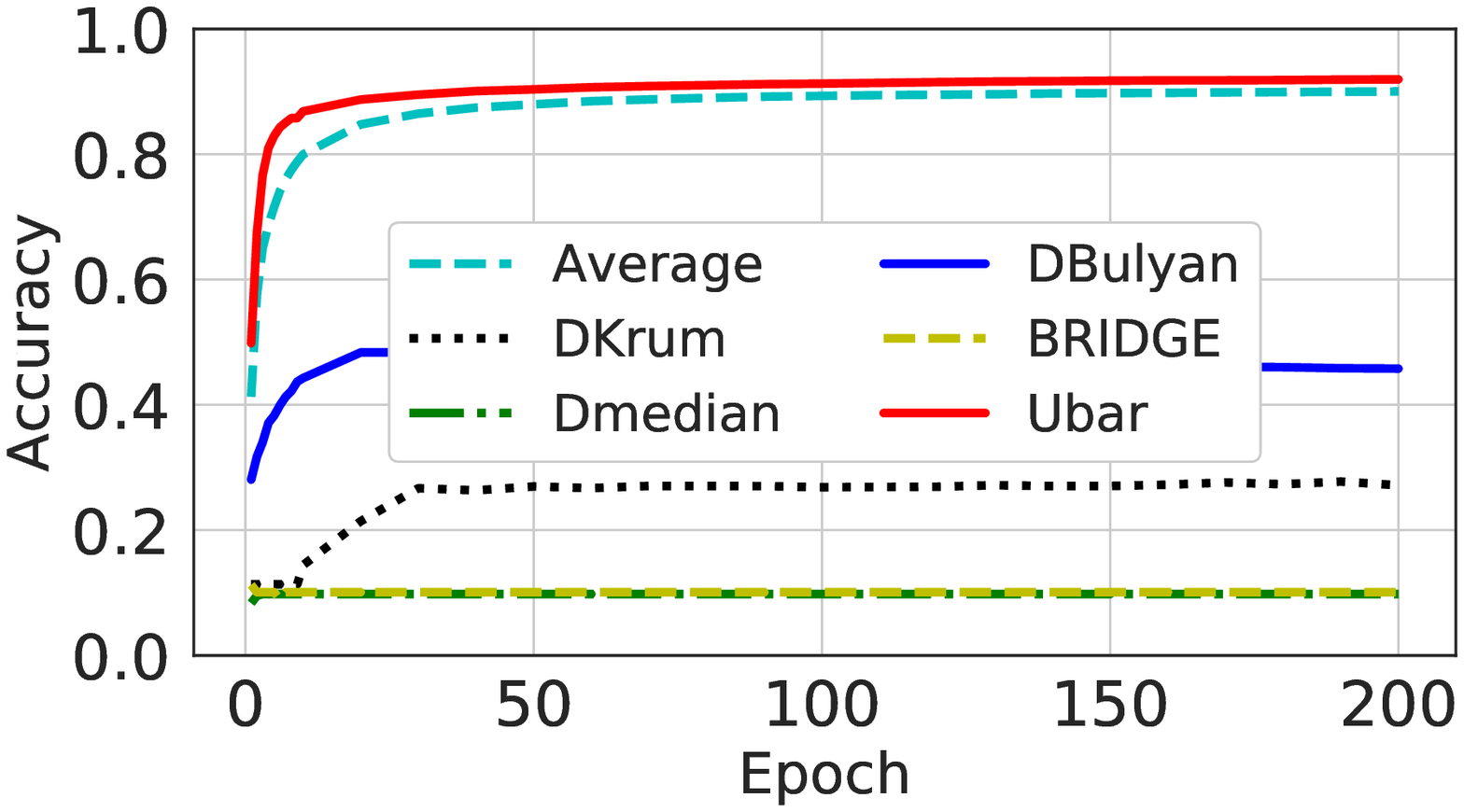}\\
			\rotatebox{90}{$\ \ \ \ \ \ \ \ \  $CIFAR10}&\includegraphics[width=0.33\textwidth]{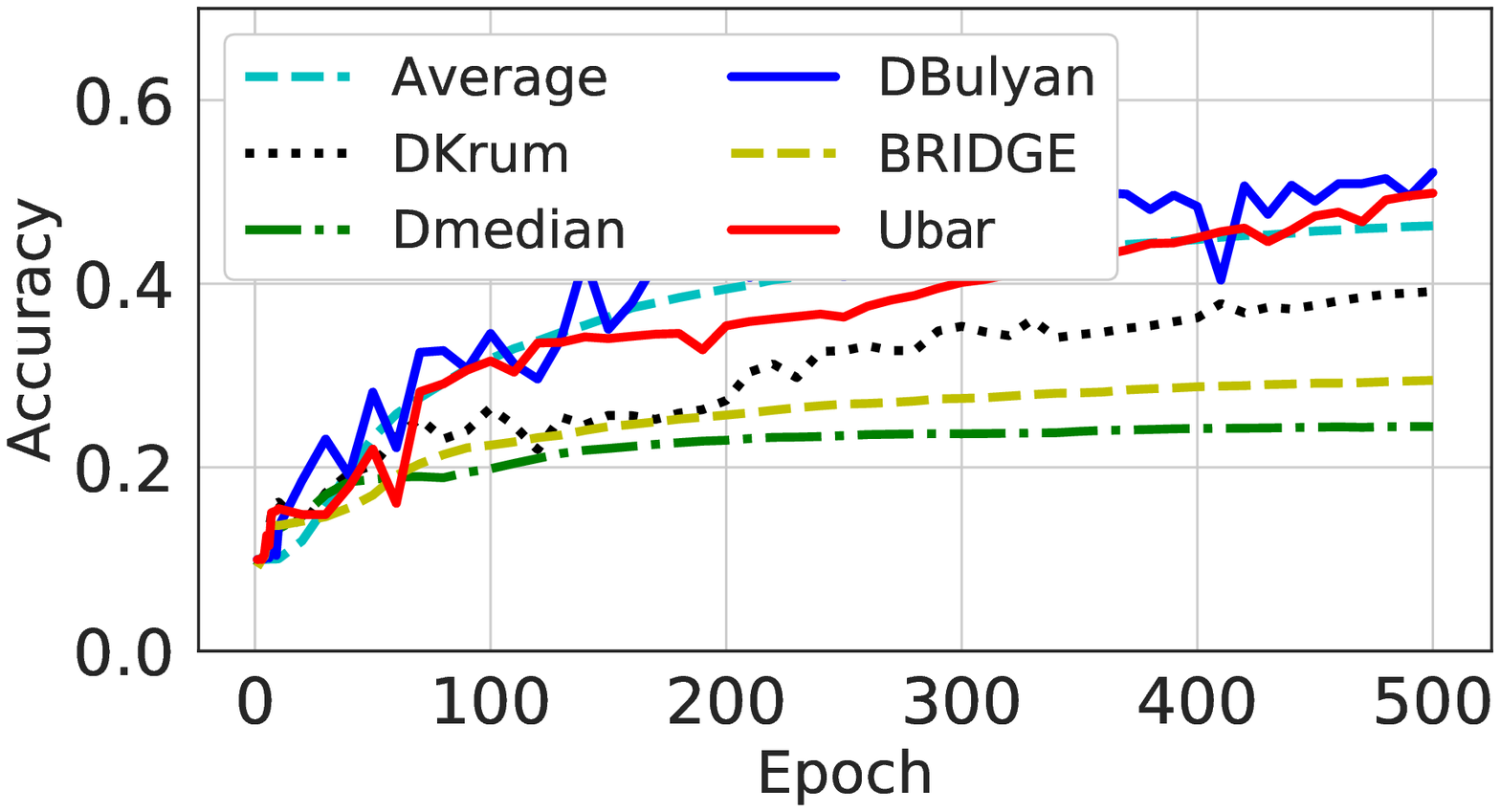}&\includegraphics[width=0.33\textwidth]{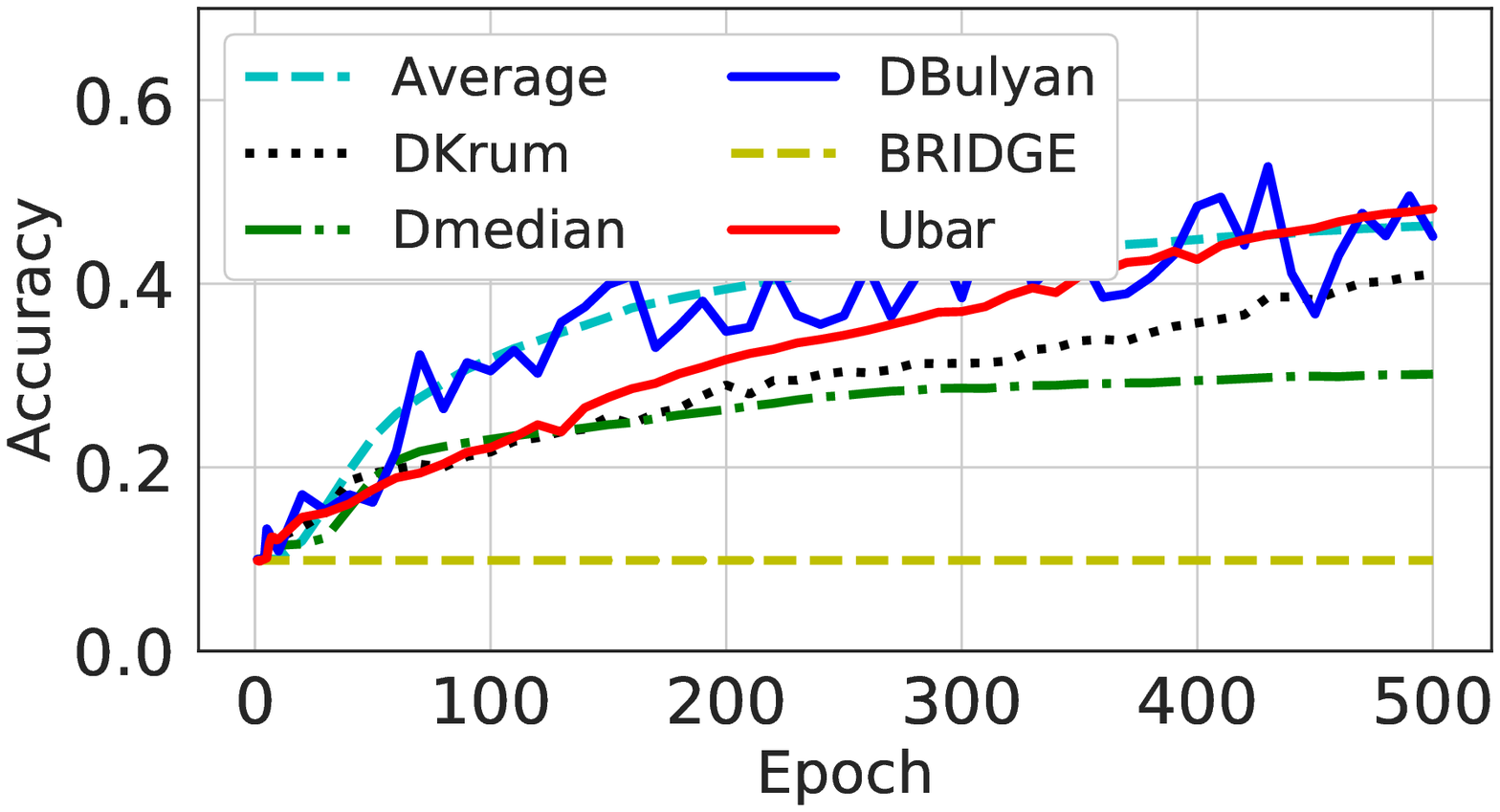}&\includegraphics[width=0.33\textwidth]{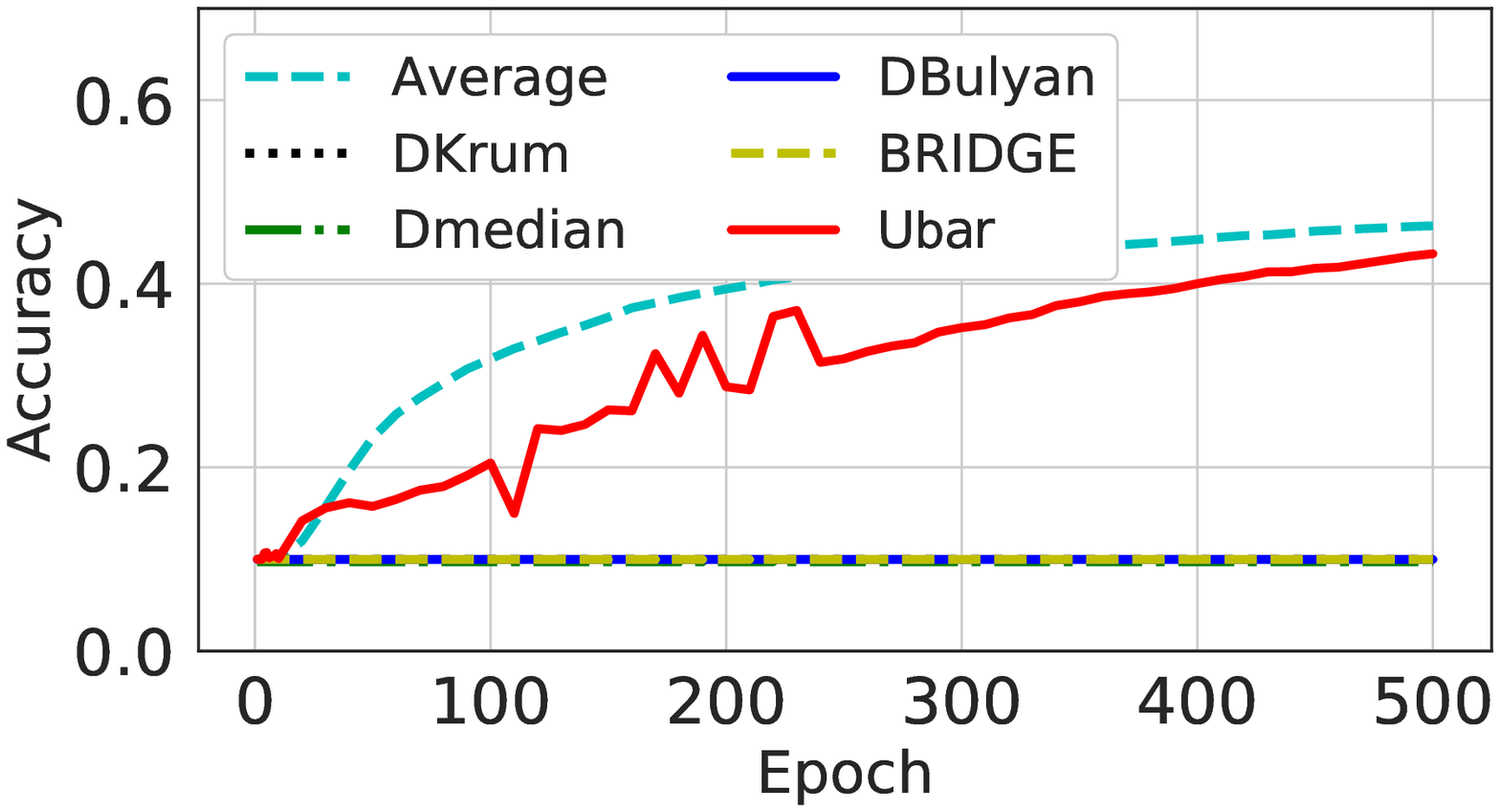}\\
			& (a) Byzantine ratio = 0.1 & (b) Byzantine ratio = 0.3& (c) Byzantine ratio = 0.5

	\end{tabular}}
	\captionof{figure}{The worst accuracy of the benign nodes under the Mhamdi attack.}\label{fig:cifar_advance_attack}
\end{table*}

\begin{table}[!htp]\centering
	\caption{The average training and aggregation time of one iteration for different aggregation rules}\label{tab:time}
	\begin{tabular}{lrrr}\toprule
	Method &Training (s) &Aggregation (s) \\
	\midrule
	Average &0.05 &0.05 \\
	Dkrum &0.05 &4.36 \\
	Dmedian &0.04 &0.26 \\
	DBulyan &0.05 &16.22 \\
	BRIDGE &0.05 &0.43 \\
	\AlgName &0.04 &0.55 \\
	\bottomrule
	\end{tabular}
\end{table}
\subsection{Convergence}
As an aggregation rule in a decentralized system, the essential functionality is to
achieve uniform convergence, i.e., the model in each benign node must converge to the
correct one. We evaluate the convergence functionality of \AlgName
with different configurations.

\bheading{Network size.}
We first evaluate the convergence of our solution under different network sizes. It is more
difficult to achieve uniform convergence when there are more nodes. In our experiments, we
consider a decentralized system with 30, 40 and 50 nodes respectively \cite{mhamdi2018hidden}. The connection ratio
is set as 0.4. Fig. \ref{fig:cifar_vary_nodes} shows the worst
accuracy during the training phase on MNIST and CIFAR10.

We can observe that only DBulyan and our proposed
\AlgName have the same convergence as the baseline on both MNIST and CIFAR10. DBulyan and \AlgName converge to a slightly better
model at a higher speed. In contrast, DKrum and BRIDGE have bad convergence performance on CIFAR10, especially when the network size is larger. Dmedian does not converge on both datasets when the network size is 50. \newtextC{We also observe that Average Aggregation Rule does not perform better than other methods even under the Byzantine-free setting. This is because this baseline needs to consider all the parameters, some of which may have poor performance even they are not malicious. In contrast, other methods selectively aggregate certain parameters with positive contributions to the model convergence, thus exhibiting better robustness.}

\bheading{Network connection ratio.}
This factor can also affect the model convergence: it takes more effort and time for all
nodes to reach the consensus when the connection is heavier. We evaluate such impact with
different connection ratio (0.2, 0.4 and 0.6), while fixing the number of nodes as 30.

Fig. \ref{fig:cifar_vary_connection} illustrates that most defense solutions in our consideration
have satisfactory convergence performance when the connection ratio is small (0.2 and 0.4). Our proposed \AlgName has better convergence performance when the connection ratio is 0.6. The BRIDGE
and DMedian approaches cannot produce correct models at this high connectivity.

\subsection{Byzantine Fault Tolerance}\label{sec:bft}
We evaluate the performance of different defense strategies under various Byzantine attacks. We set the connection ratio of the evaluated system as 0.4 and the number of
benign nodes as 30. We consider different Byzantine ratios (0.1, 0.3 and 0.5).

\bheading{Guassian attack.}
We first use a simple attack to test the Byzantine resilience. Specifically, in each iteration
the adversarial nodes broadcast to their neighbors random estimate vectors following the Gaussian
distribution. We refer to this kind of attack as Gaussian attack.

Fig. \ref{fig:cifar_simple_attack} Illustrates the model training performance under the Gaussian attack. The advantage of \AlgName over other strategies
is obvious. Dmedian, DBulyan and BRIDGE do not uniformly converge in all
systems of CIFAR10. DKrum fails to converge at the Byzantine ratio of 0.5. Only \AlgName can generate
the correct model regardless of the Byzantine ratio on both datasets.

\bheading{Bit-flip attack.}
We also implement a bit-flip attack \cite{xie2018generalized} to evaluate these defenses, where at each iteration the adversarial nodes flip the sign of the floating estimates and then broadcast these fault estimates to their neighbors.

Fig. \ref{fig:cifar_bitflip_attack} shows the convergence results: the advantage of \AlgName over other strategies
is more obvious. \AlgName has the same performance as the baseline regardless of the Byzantine ratio on both datasets. This indicates that it is absolutely Byzantine-resilient against the bit-flip attack. In contrast, Dmedian (resp. DBulyan) do not converge uniformly when the Byzantine ratio is high (0.3) on CIFAR10 (resp. MNIST). BRIDGE performs unsatisfactory on both datasets and all baselines fail to defeat the bit-flip attack when the Byzantine ratio is 0.5.

\bheading{Sophisticated attack.}
To fully evaluate the BFT of our proposed approach, we adopt a more
sophisticated attack, Mhamdi attack \cite{mhamdi2018hidden}: the adversary has the capability
of collecting all the uploaded estimates from other neighbor nodes. Then it can carefully design
its own estimate to make it undetectable from the benign ones, while still compromising the training process. Mhamdi attack has been shown effective against most defenses in centralized PS-based systems \cite{mhamdi2018hidden}.

The results are shown in Fig.
\ref{fig:cifar_advance_attack}. We observe that \AlgName can always succeed
for different Byzantine ratios on both datasets. In contrast, other solutions
fail to defeat Mhamdi attack in some cases: all solutions fail to converge when the
Byzantine ratio is 0.5. Dmedian and BRIDGE cannot converge even when the
Byzantine ratio is 0.1 on CIFAR10.


\subsection{Computation Cost}
To evaluate the computation cost of \AlgName, we measure the average training and
aggregation time for one iteration on CIFAR10. We consider a decentralized system with 30 nodes
and the connectivity is 0.4. Tab. \ref{tab:time} shows the average time of different
defense solutions. We can see that the training time for those solutions are identical, but the aggregation time differs a lot. \AlgName can finish one iteration in a much shorter time
than DKrum (8X faster) and DBulyan (30X faster). The reason is that while \AlgName only calculates the difference between a node and its neighbors, DKrum and DBulyan have to compare the distances among all neighbors. \AlgName is slightly worse than Dmedian and BRIDGE. But considering the bad convergence of Dmedian
and BRIDGE under Byzantine attacks demonstrated in Section~\ref{sec:bft}, we conclude that
\AlgName is the optimal solution with the strongest Byzantine resilience and acceptable computation overhead.

\section{Discussion}\label{sec:disscussion}
\noindent\textbf{Ratio of Benign Neighbors}. It is worth noting that each benign node needs to know the ratio of benign neighbors to perform robust aggregation. This assumption is commonly adopted in prior works \cite{blanchard2017machine,xie2019zeno,yang2019bridge}: the nodes can assess the threat of the surrounding environment before the training process and sets $\rho_i$ accordingly. In case a node does not have such knowledge, it can conservatively set $\rho_i = \frac{1}{|\mathcal{N}_i|}$, where it just assumes one benign neighbor according to Assumption 1. How to accurately estimate $\rho_i$ and aggregate the parameters in an environmental-agnostic way is beyond of the scope of our work.

\noindent\textbf{Non-IID Scenario}.
In this paper, we mainly consider the setting where all the clients use the IID data samples for collaborative training. This is consistent with other Byzantine-resilience works \cite{blanchard2017machine,xie2018generalized,yin2018byzantine,mhamdi2018hidden,xie2019zeno++,munoz2019byzantine,panjustinian,yang2019byrdie,yang2019bridge}. In reality, different clients may use non-IID training data to increase the model generalization. This will increase the difficulty of Byzantine defense, since it is hard for a node to distinguish a malicious neighbor from a benign neighbor using different distributions of samples. How to design a Byzantine-resilience method under the Non-IID scenario is a challenging task, and very few works can achieve such protection\footnote{For example, \cite{cao2020fltrust} evaluated the Non-IID case. However, this work still adopted the same training set, but just distributed unbalanced samples to different clients.}. This will be an important research direction as future work.

\section{Conclusion}~\label{sec:conclusion}
In this paper, we explore the Byzantine Fault Tolerance in decentralized learning
systems. We demonstrate that
a decentralized system is highly vulnerable to Byzantine attacks. We show that
existing Byzantine-resilient solutions in centralized PS-based systems cannot be used to
protect decentralized systems due to their security flaws and inefficiency. Then
we propose a uniform Byzantine-resilient approach, \AlgName to defeat Byzantine
attacks in decentralized learning. Experimental results reveal that \AlgName can
resist both simple and sophisticated Byzantine attacks with low computation overhead
under different system configurations.

\bibliographystyle{body/IEEEtran}
\bibliography{body/ref}

\end{document}